\documentclass{article}

\usepackage[preprint,nonatbib]{neurips_2023}

\usepackage[utf8]{inputenc} 
\usepackage[T1]{fontenc}    
\usepackage{hyperref}       
\usepackage{url}            
\usepackage{booktabs}       
\usepackage{amsfonts}       
\usepackage{nicefrac}       
\usepackage{microtype}      
\usepackage{xcolor}         
\usepackage{amsthm,amssymb,amsmath,amsfonts,empheq}
\usepackage{mathtools}
\usepackage[shortlabels]{enumitem}
\usepackage[numbers]{natbib}

\newtheorem{theorem}{\textbf{Theorem}}[section]
\newtheorem{remark}[theorem]{\textbf{Remark}}

\newtheorem{lemma}[theorem]{\textbf{Lemma}}
\newtheorem{corollary}[theorem]{\textbf{Corollary}}
\newtheorem{proposition}[theorem]{\textbf{Proposition}}

\def\aa{\alpha}

\def\dd{\delta}

\def\gg{\gamma}

\def\ll{\lambda}

\def\ss{\sigma}

\def\th{\theta}

\def\ee{\varepsilon}
\def\vp{\varphi}

\newcommand{\ddo}{{\overline \delta}}

\newcommand{\tho}{{\overline{\theta}}}

\newcommand{\cA}{{\cal A}}

\newcommand{\RR}{ \mathbb{R}}
\newcommand{\CC}{ \mathbb{C}}

\newcommand{\ZZ}{ \mathbb{Z}}

\newcommand{\EE}{{\mathbb E}}
\newcommand{\TT}{{\mathbb{T}}}

\newcommand{\cL}{{\mathcal L}}
\newcommand{\cN}{{\mathcal N}}
\newcommand{\cP}{{\mathcal P}}

\newcommand{\cS}{{\mathcal S}}

\newcommand{\cU}{{\mathcal U}}

\newcommand{\bt}{{\widetilde b}}

\newcommand{\Xt}{{\widetilde X}}

\newcommand{\ddt}{{\widetilde \dd}}

\newcommand{\tht}{{\widetilde \th}}

\newcommand{\Ct}{{\widetilde C}}
\newcommand{\Ft}{{\widetilde F}}
\newcommand{\Ht}{{\widetilde H}}

\newcommand{\Wt}{{\widetilde W}}

\DeclareMathOperator{\argmin}{\mathop{\rm argmin}}
\DeclareMathOperator{\argmax}{\mathop{\rm argmax}}

\DeclareMathOperator{\Var}{\mathop{\rm Var}}

\newcommand{\norm}[3][]{{\left\| #2 \right\|_{#3}^{#1}}}

\newcommand{\norminf}[2][]{{\left\| #2 \right\|_{\infty}^{#1}}}

\DeclareMathOperator{\divo} {div}

\newcommand{\lp}{\left(}
\newcommand{\rp}{\right)}

\newcommand{\lc}{\left\{}
\newcommand{\rc}{\right\}}
\newcommand{\lb}{\left[}
\newcommand{\rb}{\right]}

\newcommand{\labs}{\left|}
\newcommand{\rabs}{\right|}

\newcommand{\tr}{\mbox{{\rm tr}}}

\title{Temporal Difference Learning with Continuous Time and State
           in the Stochastic Setting}

\author{%
  Ziad Kobeissi\\
  Inria \& \'Ecole Normale Supérieure\\
  Institut Louis Bachelier, Paris, France\\
  \texttt{ziad.kobeissi@inria.fr} \\
   \And
  Francis Bach\\
  Inria \& \'Ecole Normale Supérieure\\
  PSL Research University, Paris, France\\
  \texttt{francis.bach@inria.fr} \\
}

\begin{document}

\maketitle

\begin{abstract}
    We consider the problem of 
    continuous-time policy evaluation.
    This consists in learning through observations
    the value function associated
    to an uncontrolled
    continuous-time stochastic dynamic and a reward function.
    We propose two original variants of the
    well-known TD(0) method
    using vanishing time steps.
    One is model-free and
    the other is model-based.
    For both methods,
    we prove theoretical convergence rates
    that we subsequently verify through numerical simulations.
    Alternatively, 
    those methods can be interpreted
    as novel reinforcement learning approaches
    for approximating solutions of
    linear PDEs (partial differential equations)
    or linear BSDEs (backward stochastic differential equations).
\end{abstract}

\section{Introduction}
\label{sec:intro}
Consider the value function $V$
obtained from a continuous time
and state stochastic process $(X_t)_{t\geq0}$,
\begin{equation}
    \label{eq:def_V}
        V(x)
        =
        \EE\lb
        \int_0^{\infty}
        e^{-\rho t}r(X_t)\,dt
        \Big|X_0=x
        \rb
        \quad\text{ with }\quad
        dX_t
        =
        b(X_t)dt
        +\sigma(X_t) dW_t,
\end{equation}
where $r,b,\ss:\Omega\to
\RR,\RR^d,\RR^{d\times d_W}$
are respectively the reward,
drift and diffusion functions,
$\rho>0$ is the exponential
discount rate
and $W$ is a $d_W$-dimensional
Brownian motion.
The state $(X_t)_{t\geq0}$
is a continuous stochastic process
valued in a continuous state
space $\Omega$
and satisfies a stochastic differential
equation (SDE).
For readers who are not familiar
with SDEs,
a heuristic interpretation of
\emph{``$X$ is a solution to the
    right-hand side of \eqref{eq:def_V}''}
could be:
given $X_t\in\Omega$ at time $t\geq0$,
between $t$ and $t+dt$
the state is updated by adding
a deterministic term given by
$b(X_t)\,dt$
and a stochastic term of the form
$\sqrt{dt}\,\ss(X_t)\xi$ with
$\xi\approx\cN(0,I_d)$ being independent
of all previous states.
This is the natural extension to continuous time
of Markov Chains with Gaussian transitions
(which are widely used in literature
on stochastic optimisation).
For an introduction to the theory of SDE,
we refer to \cite{kloeden1992stochastic}.

The purpose of the present work
is to analyse methods for
learning the continuous-time
value function~$V$
from discrete-time observations of the
state and the reward,
using vanishing time steps.
Therefore, this enters the field
of reinforcement learning (RL)
and more precisely the well-known
problem of \emph{policy evaluation}.
However, most of the theory of RL
is fundamentally stated with fixed discrete time.
Few extensions to continuous-time problems
exist in the literature,
see \cite{dayan1996improving,doya2000reinforcement,lutter2021value}
for instance,
but we are not aware of any
theoretical research work
dealing with the stochastic setting (i.e., when $\ss\not\equiv 0$).
Such studies are yet of the utmost importance
given that: first, a large part of the
modern applications of RL come
from time discretisations of continuous-time
problems;
second, the constant improvement
of the available computational power makes it possible
to consider finer and finer time discretisations
(leading to the need of new numerical methods
which are robust to decreasing the time steps);
third, working in the deterministic set-up is
not sufficient since
deterministic processes are
never considered in practice in RL
(even when the model is deterministic,
stochastic perturbations are always added
in practice to favor exploration,
see Appendix \ref{subsec:noises}
for more details).
The present work is the first, up to our knowledge,
to investigate this direction
from a theoretical perspective.

The results proved in the present work
emphasise the fact that the presence of
a stochastic part in the dynamics leads
to new technical difficulties
when the time steps are small.
Therefore, standard RL algorithms
have to be properly adapted.
This also leads to new
theoretical analyses.
Like usually,
most of such
an analysis
is only
possible through the restrictive
assumption of 
linear parametrisations.

\subsection{Definitions of the main objects}
We approximate $V$ defined in \eqref{eq:def_V}
by a parametrised function
$x\mapsto v(x,\th)$,
where $\th\in\Theta$ is the learnt
parameter and $\Theta$ is the space of parameters
(usually given by $\RR^{d_{\theta}}$
for some positive integer $d_{\theta}$).

\paragraph{Definitions of the temporal differences.}
Here,
an observation is a quadruple
$(\Delta t,X,X',R)$,
where $\Delta t>0$
is deterministic, 
$X$ and $X'$ are two states
arising at some times
$t$ and $t+\Delta t$,
and $R$ is the running reward
between those times.
See Section \ref{sec:assumptions}
for the precise assumptions
on the observations.
From that,
one may usually compute
$\dd$
the \emph{standard temporal difference}
defined by
\begin{equation}
    \label{eq:def_dd}
    \dd_{\Delta t}
    =
    \textstyle{\frac1{\Delta t}}
    \lp
    v(X,\th)
    -\gamma_{\Delta t}v(X',\th)
    -\Delta t\,R
    \rp,
\end{equation}
where $\gamma_{\Delta t}=e^{-\rho\Delta t}$,
and the scaling $\frac1{\Delta t}$
is chosen so that
the expectation of $\dd$ 
is of order $O(1)$.
This quantity
was initially designed
for discrete-time dynamics
and is derived from
Bellman's programming principle
\cite{bellman1966dynamic}.
In the present work,
one of the main contributions
is to introduce
and analyse
$\ddt$
the \emph{stochastic temporal difference},
\begin{equation}
    \label{eq:def_ddt}
    \ddt_{\Delta t}
    =
    \dd_{\Delta t}
    +\textstyle{\frac1{\Delta t}}Z
    \quad\text{ with }\quad
    Z
    =
    \lp X'
    -X
    -\Delta t\,b(X)\rp
    \cdot
    \nabla_xv(X,\th).
\end{equation}
The two temporal differences $\dd$ and $\ddt$
only differ by the additional term
$Z$, that here is called
\emph{the (stochastic) correction term}
or \emph{the variance-reduction term}.
Observe that computing the stochastic TD 
requires to know
the drift function $b$ at any observation,
but not $\sigma$.
This means that
the expected direction of the dynamics
has to be known,
but not the law of the noises
(which does not have to be gaussian,
see Section \ref{subsec:limitations}).
Such an assumption
is standard when
the physic of the model is known
but some uncertainties may generate
noises,
like measurement noises,
model approximations
(e.g., statistical or
discretisation errors),
unpredictable external factors
(e.g., imperfections of the
ground or wind for robotic)
or idiosyncratic noises
(e.g., in financial models).
In practice, 
stochastic TD is more meaningful for small time steps,
but standard TD can be used in model-free environments.

\paragraph{Definitions of the Algorithms.}
We denote $\dd_k$ (resp. $\ddt_k$)
as the standard (resp. stochastic) TD
at iteration $k$, i.e., computed
with $(\Delta t_k,X_k,X'_k,R_k)$
the $k^{\text{th}}$ observation.
For a sequence of learning rates $(\alpha_k)_{k\geq1}$,
the standard and stochastic TD(0) methods
are defined by
\begin{equation}
    \label{eq:TD0}
    \tag{TD0}
    \begin{aligned}
        \th_{k+1}
        &=
        \th_k
        -\aa_k\dd_{k}\nabla_{\th}v(X_k,\th_k),
        \\
        \tht_{k+1}
        &=
        \tht_k
        -\aa_k\ddt_{k}\nabla_{\th}v(X_k,\tht_k).
    \end{aligned}
\end{equation}
We also introduce the regularised
standard and stochastic 
TD(0) methods as
\begin{equation}
    \label{eq:muTD0}
    \tag{$\mu$-TD0}
    \begin{aligned}
        \th_{k+1}
        &=
        \Pi_{B_M}\big(
        \th_k
        -\aa_k(\dd_{k}\nabla_{\th}v(X_k,\th_k)
        +\mu\th_k)\big),
        \\
        \tht_{k+1}
        &=
        \Pi_{B_M}\big(
        \tht_k
        -\aa_k(\ddt_{k}\nabla_{\th}v(X_k,\tht_k)
        +\mu\tht_k)\big),
    \end{aligned}
\end{equation}
where $\mu\geq0$ and 
$\Pi_{B_M}$ is the projection
on $B_M$ the Euclidean ball of $\RR^d$
centered at $0$ with radius $M$.

\subsection{Motivations}
\label{subsec:motivations}
In the second paragraph of this text,
we already motivated the
need for theoretical studies 
of RL methods which are robust
with respect to the decrease of time steps
in the stochastic setting.
In this section, we give more details on the
precise types of problems that can be treated
using our methods.

\paragraph{Optimal control.}
Observe that there is no control
in \eqref{eq:def_V} and that controlled
dynamical systems will not be considered in the main text
of the present work (only in Appendix \ref{sec:appli_RL}).
Still, our main motivation is to solve optimal control problems
in high-frequency regimes.
Let us point out that a large part of the
modern RL numerical methods are using
TD(0) iterations (or basic extensions of it)
as a subroutine for doing \emph{policy evaluation} (PE).
Those methods are then completed by adding
a \emph{policy improvement} (PI) process
which uses the current approximation of the value function
(or the $Q$-function) to improve the control function.
For more details we refer to \cite{MR3889951}
and the class of \emph{generalised policy iterations}.
We believe that the main difficulty for extending most of the
RL numerical methods to the high-frequency regime
precisely lies in the extension of the PE process,
since PI often boils down to
approximating an \emph{argmax} operator
(see \cite{lutter2021value} for instance)
which is generally less related to the dynamics.
This and the fact that the assumptions and analysis for PI
are different from the ones of PE
explain why we only consider
PE in the present work.
We refer to Appendix~\ref{subsec:PI}
for an informal discussion on
one way to extend our
numerical methods to PI as well.

\paragraph{Solving partial or stochastic
    differential equations (PDEs or SDEs).}
See Appendix \ref{sec:app_notations}
for some standard notations from the PDE literature.
The Feynman-Kac
formula states
that $V$ from \eqref{eq:def_V} satisfies
\begin{equation}
    \label{eq:PDE_V}
    r
    =
    \cL V
    :=
    \rho V
    -\tr\lp\textstyle{\frac{\ss\ss^{\top}}{2}}D^2_{x,x}V\rp
    -b\cdot\nabla_xV.
\end{equation}
Alternatively,
the couple $(Y_t,Z_t):=(V(X_t),\sigma(X_t)^{\top}\nabla_xV(X_t))$
solves the backward SDE (BSDE)
\begin{equation}
    \label{eq:BSDE_Y}
    dY_t
    =
    -(r(X_t)-\rho Y_t)\,dt
    +Z_t\cdot dW_t,
\end{equation}
where we recall that
solutions of BSDEs are couples of stochastic processes
(unlike standard forward SDEs whose solutions
are scalar processes),
see \cite{peng1993backward} for an introduction to BSDEs
and their application to optimal control.
Consequently,
another motivation is that 
the methods analysed 
in the present work
can be viewed as original
numerical methods
for solving linear PDEs and BSDEs
of the forms given above
using only discrete observations of the dynamics.
Using similar arguments as in the previous
paragraph on optimal control,
our methods can be extended to solve
nonlinear PDEs and BSDEs
but a theoretical analysis of such systems
is out of the scope of the present work.
For instance, adding a PI process allows to solve
Hamilton-Jacobi equations and
forward-backward SDE systems (FBSDE) arising in stochastic
control problems (see Appendix~\ref{sec:appli_RL}).
More general nonlinear problems can be treated as well
using other additional processes, like Picard iterations
for instance, see \cite{chen2021learning}.

\paragraph{Some real-world examples.}
Giving an exhaustive list of the applications
for approximating
quantities from \eqref{eq:def_V},
\eqref{eq:PDE_V}, \eqref{eq:BSDE_Y}
and their nonlinear counterparts 
is out of reach since it covers
a very large number of stochastic 
continuous-time problems
(with and without control).
Moreover, we believe that this number
is only going to grow faster in the future
as the available computational power is still growing,
the discretisation steps become finer and finer in practice
and the interest on learning methods is growing 
inside communities working with continuous-time
problems (using PDEs and SDEs for instance).
Here, we give only
three examples in which
high-frequency RL are already studied in the literature.
First, the robotics in real time,
e.g., \cite{lidec2022leveraging},
where the physic of the models
is well known (so is the drift function~$b$)
but some external phenomenon can
only be implemented through additional
noise (e.g.,
wind, imperfections of the ground
or measurement noises).
Second, modern finance models
are often based on SDEs,
see \cite{el1997backward}
for a thorough introduction,
we refer to the surveys
\cite{germain2021neural,hu2023recent}
for machine learning methods
to solve BSDEs arising in finance~;
for instance, an application
is the high-frequency trading,
where the stochastic part comes from
idiosyncratic noises.
Third, 
models of
nuclear fusion in tokamaks,
e.g., \cite{degrave2022magnetic},
requiring
high-frequency controls
of the magnetic field,
while noises may come from
measurement errors or
simplifications
of the physical model
which is too complicated to be accurately
implemented with limited computational power.

\subsection{Main contributions and originalities.}
\label{subsec:contributions}
The main contributions of the present work
are:
\begin{itemize}
    \item 
        For general parametrisation
        of the learnt value function
        (i.e., linear or non-linear),
        we show that
        the standard TD introduced in \eqref{eq:def_dd}
        is not suited to high-frequency regimes
        (since its variance blows up),
        while its stochastic counterpart
        from \eqref{eq:def_ddt} is.
        We refer to Proposition~\ref{prop:varred}.
    \item
        We introduce two original variants of the TD(0) method
        for estimating continuous-time value functions
        or for numerically approximating the solutions of
        linear PDEs or BSDEs.
        One is model-free and the other is model-based.
    \item
        Recall
        that convergence rates
        only exist in the literature under
        the restrictive assumption of a
        linear parametrisation.
        Under such an assumption,
        we prove that standard TD(0)
        converges for some decreasing sequence
        of time steps.
        This is surprising
        when compared with the literature
        on BSDEs since it may be interpreted as:
        the solution of \eqref{eq:BSDE_Y} can be learnt
        via an iterative method which totally omits
        the stochastic term ``$Z_t\cdot dW_t$'',
        see Remark \ref{rk:simple_cvg}.
    \item
        Under a linear parametrisation,
        we prove that stochastic TD(0)
        converges, with a faster convergence rate
        and is more robust with respect
        to the choice of the time steps
        than standard TD(0).
        More precisely, the speed of convergence
        is similar to the one of the simpler algorithm SGD for linear regression
        (up to the state of the art for Theorem \ref{thm:averaged_cvg}).
    \item
        We show numerical simulations with similar
        convergence rates as predicted by the theory.
    \item
        We give an original interpretation
        of TD(0) as a minimising method
        under a particular structure of
        the dynamics (for a constant $\ss$ and
        a drift of the form $b=\nabla_xU$),
        see Proposition~\ref{prop:langevin}.
\end{itemize}

\section{Assumptions}
\label{sec:assumptions}
\paragraph{Assumption on the informations.}
Let us explain more the assumptions
on the observations used for computing
\eqref{eq:TD0} and \eqref{eq:muTD0}.
A sequence $(\Delta t_k,X_k,X'_k,R_k)_{k\geq1}$
of observations is such that
$\Delta t_k>0$ is convergent to zero
and $(X_k,X'_k,R_k)$ are
independent
random variables which can be:
\begin{itemize}
    \item
        {\bf
        Real-world observations:}
        $X_k$, $X_k'$ and $R_k$ are
        the real states and reward
        obtained using
        the continuous-time dynamics
        in \eqref{eq:def_V}
        on a time interval of length $\Delta t$,
        i.e.,
        \begin{equation*}
            (X_k,X'_k,R_k)=\lp \Xt_0,\Xt_{\Delta t_k},
            \frac1{\Delta t_k}\int_0^{\Delta t_k}r(\Xt_t)\,dt\rp
            \,\text{ where }\,
        d\Xt_t
        =
        b(\Xt_t)dt
        +\sigma(\Xt_t) dW_t.
        \end{equation*}
    \item
        {\bf
        Observations from a simulator:}
        $X_k$, $X'_k$ and $R_k$
        are obtained using the
        Euler-Maruyama discretisation
        scheme of the SDE in \eqref{eq:def_V}.
        The step operator is
            $\cS_{\Delta t}:(x,z)\mapsto
            x
            +\Delta t\,b(x)
            +\sqrt{\Delta t}\,
            \ss(x)
            z$,
        then
        we assume
        $X'_k=\cS_{\Delta t_k}(X_k,\xi_k)$ and
        $R_k=r(X_k)$.
\end{itemize}

\paragraph{Assumption on the law of the observations.}
We denote the law of $X_k$ by $m_k$,
recall that $m$
denotes the stationary
measure of the dynamics
\eqref{eq:def_V}.
We assume that 
$m_k$ is convergent to $m$
in the sense of distributions
and that 
there exists a
nonnegative integer $p$,
such that
\begin{enumerate}[label=\bf{A\arabic*}]
    \item
    \label{eq:cvg_mk}
for any $f\in C^{p}(\Omega;\RR)$,
there exists $C_f>0$ such that
    $\labs
    \EE\lb
    f(X_k)
    -f(X)
    \rb
    \rabs
    \leq
    C_f\Delta t_k,$
for $k\geq0$,
where $X$
is distributed
according to $m$.
\end{enumerate}
In practice,
such a condition
is obtained
using ergodic arguments while
following the Markov Chain
(which can be continuous or discrete depending
whether the observations are from the real world
or from a simulator);
see Appendix
\ref{subsec:discrete_MC}
and Theorem \ref{thm:cvg_mk}
for such a result
with $p=4$.

\paragraph{Choice of the boundary conditions.}
In discrete state space,
boundary conditions are in general
missing and unnecessary since the
Markov transition probability
is naturally designed such that
the trajectories stay inside
the state space,
or may only leave through
specific terminal states.
In continuous time and state,
and especially in stochastic settings,
things become much more complicated.
Indeed, boundary conditions of different
natures appear naturally when establishing
the models, each involving different
theoretical and numerical difficulties.
For instance,
homogeneous Neumann conditions
correspond to reflexive walls,
Dirichlet boundary conditions
correspond to exits,
periodic boundary conditions
are used for modeling some standard 
non-euclidean geometries
(like spherical or cylindrical   
coordinates)
and others like
mixed Robin conditions
or state constraints
may correspond to other
physical considerations;
we refer to \cite{evans2010partial}
for more details.
Those conditions may even differ
on different parts of the boundary
or different dimensions.
For those reasons, we have to make
a choice;
if our model does not 
cover all physical aspects
of continuous dynamics,
it is complex enough
to capture the main
ideas for extending
RL methods to continuous models.
We decide to only consider periodic
boundary conditions
and a state space given
by the $d$-dimensional torus, i.e.,
$\Omega=\TT^d=\RR^d/\ZZ^d$.
Nevertheless, we argue that
adapting our arguments to different
boundary conditions is totally feasible
but out of the scope of the present
paper (since it would lead
to unnecessary technical difficulties
that we prefer to avoid
for this work to stay as simple as possible).

\section{Preliminary results in the general case}
\label{sec:TDs}
The quantity
$\EE[\delta_{\Delta t}\big| X]$
(resp. $\EE[\ddt_{\Delta t}\big| X]$
for stochastic TD)
is usually dubbed the Bellman error.
This quantity is equal to zero
if $v(\cdot,\th)$ is exactly the
value function (of the discrete-time dynamics).
Therefore,
a standard way to
check that the learnt function
$v(\cdot,\th)$
is a good approximation
of the value function $V$,
is to check if the Bellman error
is near zero in some sense for small time steps.
However, the Bellman error
is not convenient
to compute from observations,
nor is useful for
learning
(since for any fixed value of $X$,
a lot of observations
are needed).
A common alternative which is easier to compute
is the average TD squared,
it admits the following
decomposition,
\begin{equation}
    \label{eq:mean_sq_TD}
    \underbrace{\EE_{(X,X')}\big[ | \dd_{\Delta t} |^2\big]}_{\text{averaged TD squared}}
    =
    \EE_X\big[{\underbrace{
    \EE_{X'}\big[\dd_{\Delta t}\big| 
    X\big]}_{\text{Bellman error}}}^2\big]
    + \underbrace{\EE_X\big[\Var_{X'}\big( \dd_{\Delta t}
        \big| X\big)\big]}_{\text{perturbating term}}.
\end{equation}
The following proposition 
allows a simple comparison of the
orders of magnitude in the latter equality
and its counterpart with $\ddt$,
showing that $\ddt$ fits better
than $\dd$
in the high-frequency regime.
\begin{proposition}
    \label{prop:varred}
    Assume that $r$, $b$ and $\ss$
    are bounded,
    and that $v$
    admits bounded
    continuous
    derivatives in $x$ 
    everywhere up to order
    two.
    The means and variances of $\dd$ and $\ddt$
    given $X$
    satisfy,
    \begin{align*}
        &\lim_{\Delta t\to0}
        \EE_{X'}[\dd_{\Delta t}|X]
        =
        \lim_{\Delta t\to0}
        \EE[\ddt_{\Delta t}|X]
        =
        \cL v(X,\th)
        -r(X),
        \\
        &\lim_{\Delta t\to0}
        \Delta t\Var_{X'}(\dd_{\Delta t}|X)
        =
        \labs \ss(X)\nabla_xv(X,\th)\rabs^2,
        \\
        &\lim_{\Delta t\to0}
        \Var_{X'}\lp\ddt_{\Delta t}|X\rp
        =
        2\tr\lp(\ss\ss^{\top} D^2_xv(X,\th))^2\rp.
    \end{align*}
\end{proposition}
In the one hand,
because the variance diverges as $\frac{1}{\Delta t}$,
the latter proposition directly implies
that the perturbating term
in~\eqref{eq:mean_sq_TD}
should converge to infinity
when the time step is small;
thus it would
totally overwhelm the interesting term.
On the other hand,
this does not happen when $\ddt$ replaces
$\dd$ since, in this case, the perturbating term
remains bounded.
More precisely,
at the limit
$\Delta t\to0$,
we get
\begin{align*}
    &\lim_{\Delta t\to 0}
    \EE\lb|\dd_{\Delta t}|^2\rb
    =
    \left\{
        \begin{aligned}
            &+\infty
            \text{ if }
            v(\cdot,\th) \text{ is not constant,}
            \\
            &\EE\big[\lp \rho C
            -r(X)\rp^2\big]
            \text{ if }
            v(\cdot,\th)=C. 
        \end{aligned}
    \right.
    \\
    &\lim_{\Delta t\to 0}
    \EE\big[|\ddt_{\Delta t}|^2\big]
    =
    \EE_X\lb(\cL v(X,\th)-r(X))^2\rb
    +2\EE_X\lb\tr\lp(\ss\ss^{\top} D^2_xv(X,\th))^2\rp\rb.
\end{align*}
An interesting consequence of the former
equality
is that, at the limit $\Delta t\rightarrow0$,
SGD (or any other gradient descent method)
applied to \eqref{eq:mean_sq_TD}
can only converge to a constant.
Instead for stochastic TD,
the latter equality only implies that
SGD would converge to a biased limit,
where the bias can be small in practice.
In Appendix \ref{sec:RG}
we give more details on this method,
named \emph{residual gradient}
\cite{baird1995residual},
we prove convergence results
and some alternatives to reduce the
bias are discussed in Section
\ref{subsec:alternatives_RG}.

    The latter paragraph
suggests that
TD(0) is not a
stochastic gradient method
(otherwise we would not be able
to prove convergence results
as we do in the next section).
Indeed, it is only
a stochastic \emph{semi-gradient}
method
because the term
$\dd\nabla_{\th}v$ in \eqref{eq:def_dd}
is not a gradient in general.

\section{Convergence Results in the Linear Setting}
\label{sec:lin}
From here on,
we assume \eqref{eq:cvg_mk}
and the following
assumptions:
\begin{enumerate}[label=\bf{A\arabic*}]
        \addtocounter{enumi}{1}
    \item
        \label{hypo:v_lin}
        The function $v$
        is linear with respect to $\th\in\RR^{d_{\th}}$,
        i.e.,
        $v(x,\th)=\th^{\top}\vp(x)$
        where $\vp:\Omega\rightarrow\RR^{d_{\th}}$.
    \item
        \label{hypo:rb_reg}
        The functions $r$, $b$, $\ss$
        are $C^p$,
        where $p$ comes from
        \eqref{eq:cvg_mk}
        and $b$ is at least Lipschitz continuous.
    \item        
        \label{hypo:vp_reg}
        The feature vector $\vp$
        is $C^{p+2}$,
        its coordinate functions
        are linearly independent.
\end{enumerate}

\subsection{Identification of the limits}
\label{subsec:limits}
The eventual limit of \eqref{eq:TD0}, named $\th^*$,
is given by
\begin{equation}
    \label{eq:def_th*}
    \EE_{m}\lb\vp(X)\cL \vp(X)^{\top}\rb\th^*
    =
    \EE_{m}\lb r(X)\vp(X)\rb.
\end{equation}
The linear independence assumption
on the coordinate functions of $\vp$
in Assumption \ref{hypo:vp_reg}
implies that 
$\EE\lb\vp(X)\vp(X)^{\top}\rb\in\RR^{d_{\th}\times d_{\th}}$
is positive definite.
Therefore, Lemma \ref{lem:HSA} in the Appendix
implies that
the symmetric part of 
$\EE\lb\vp(X)\cL \vp(X)^{\top}\rb$
is positive definite
as well.
This implies that
$\th^*$ is well defined
and that there exists $M_0>0$
such that $|\th^*|\leq M_0$.
Note that the latter quantity is a 
convenient choice
for $M$ in \eqref{eq:muTD0},
especially
when $\mu$ is small.
When $\mu$ is not small,
we prefer the simpler quantity
$M_{\mu}=
\mu^{-1}\norminf{r}$.
For \eqref{eq:muTD0},
we always assume
that
\begin{equation}
    \label{eq:def_M}
    M
    \geq
    \min(M_0,M_{\mu}).
\end{equation}
In this case,
$\th^*_{\mu}$,
the eventual limit
of \eqref{eq:muTD0},
satisfies
$|\th^*_{\mu}|\leq M$
so it is independent of $M$
because it remains unchanged by 
the projection step $\Pi_{B_M}$.
It is then given by
\begin{equation}
    \label{eq:def_th*mu}
    \lp\mu I_d+\EE_{m}[\vp(X)\cL \vp(X)]\rp\th^*_{\mu}
    =
    \EE_{m}\lb r(X)\vp(X)\rb.
\end{equation}
The bias induced by the regularisation
    is the distance between $\th^*$
    and $\th^*_{\mu}$,
    it is bounded as follows.
\begin{proposition}
    \label{prop:dist_th*mu}
    There exists $C>0$ such that, for any $\mu>0$, we get
        $|\th^*-\th^*_{\mu}|\leq~C\mu$.
\end{proposition}
If $V$ belongs to the set of parametrised functions,
i.e., there exists $\th_V$ such that $V=v(\cdot,\th_V)$,
then \eqref{eq:def_th*} implies that $\th^*=\th_V$,
so \eqref{eq:TD0} would converge to $V$ eventually.
Otherwise, $v(\cdot,\th^*)$ is known to be a good estimator
of $V$ in practice but cannot be expressed as
the minimiser of a meaningful functional.
Yet, we succeed to do so below
under particular structure assumptions on the dynamics.
\begin{proposition}
    \label{prop:langevin}
    Assume that $\ss$ is constant
    and $b$ is of the form $b=\nabla_xU$
    for some continuously differentiable function $U:\Omega\to\RR$,
    then $\th^*$ is the solution of the following
    minimisation problem,
    \begin{equation*}
        \label{eq:def_th*_lang}
        \th^*
        \in
        \argmin_{\th}
        \EE_{X\sim m}\lb
        \ell(v(X,\th),V(X))
        \rb
        \;\text{ where }\;
        \ell(v,w)
        =
        \rho(v-w)^2
        +\textstyle{\frac{1}{2}}
        |\ss^{\top}\nabla_x(v-w)|^2.
    \end{equation*}
\end{proposition}

\subsection{The regularised TD(0)}
\label{subsec:simple_cvg}
In this section, we only consider
the regularised algorithms \eqref{eq:muTD0}.
Under a common decreasing assumption
on the learning rate (that it is proportional
to $1/(\mu(k+1))$)
and convenient choices on $\Delta t_k$,
Theorem \ref{thm:simple_cvg} below
states a usual convergence rate in $1/k$
for the stochastic TD(0) method.
This rate can easily be compared to the literature
e.g., \cite{sutton1988learning}.
For the standard TD(0),
we obtain a slower convergence rate
because of $\dd$ 
being not adapted
to small time steps
(as explained in Section~\ref{sec:TDs}).
\begin{theorem}
    \label{thm:simple_cvg}
    Take
    $(\th_k)_{k\geq0}$
    and
    $(\tht_k)_{k\geq0}$
    defined by \eqref{eq:muTD0}
    with
    $\mu>0$,
    $M$ satisfying~\eqref{eq:def_M},
    and $\aa_k=\frac{2}{\mu(k+1)}$.
    There exists $C>0$
    such that, for $k\geq1$,
    
    \vspace*{-0.5cm}
    \raisebox{2ex}{
    \begin{minipage}[b]{0.49\textwidth}
    \begin{equation*}
        \EE\lb
        \labs\th_k-\th^*_{\mu}\rabs^2\rb
        \leq
        \frac{C}{\mu^2k^{\frac23}}
        \;\text{ for }\;
        \Delta t_k= \textstyle{\frac{1}{(k+1)^{\frac13}}},
    \end{equation*}
    \end{minipage}}
    \rule{0.5mm}{0.8cm}
    \raisebox{2ex}{
    \begin{minipage}[b]{0.49\textwidth}
    \begin{equation*}
        \EE\lb
        \labs\tht_k-\th^*_{\mu}\rabs^2\rb
        \leq
        \frac{C}{\mu^2k}
        \;\text{ for }\;
        \Delta t_k\leq \textstyle{\frac{1}{\sqrt{k+1}}}.
    \end{equation*}
    \end{minipage}}
\end{theorem}
Recall that $\th_{\mu}^*$
is a biased limit
whose bias is bounded
from Proposition \ref{prop:dist_th*mu}.
Combined with Theorem~\ref{thm:simple_cvg},
we obtain the following
approximation results
of $\th^*$ using \eqref{eq:muTD0}.
\begin{corollary}
    \label{cor:Non_Asymp}
    Under the same assumption as in Theorem \ref{thm:simple_cvg},
    after $K\geq2$ iterations,
    we have 
    
    \raisebox{1.5ex}{
    \begin{minipage}[b]{0.49\textwidth}
    \begin{equation*}
        \labs\th_K-\th^*\rabs^2
        \leq
        \textstyle{\frac{C}{K^{\frac13}}}
        \;\text{ for }\;
        \mu=K^{\frac16},
    \end{equation*}
    \end{minipage}}
    \rule{0.5mm}{0.6cm}
    \raisebox{1.5ex}{
    \begin{minipage}[b]{0.49\textwidth}
    \begin{equation*}
        |\tht_K-\th^*|^2
        \leq
        \textstyle{\frac{C}{\sqrt{K}}},
        \;\text{ for }\;
        \mu=K^{\frac14}.
    \end{equation*}
    \end{minipage}}
\end{corollary}

\begin{figure}
	\centering
    \noindent
	\includegraphics[width=0.49\linewidth]{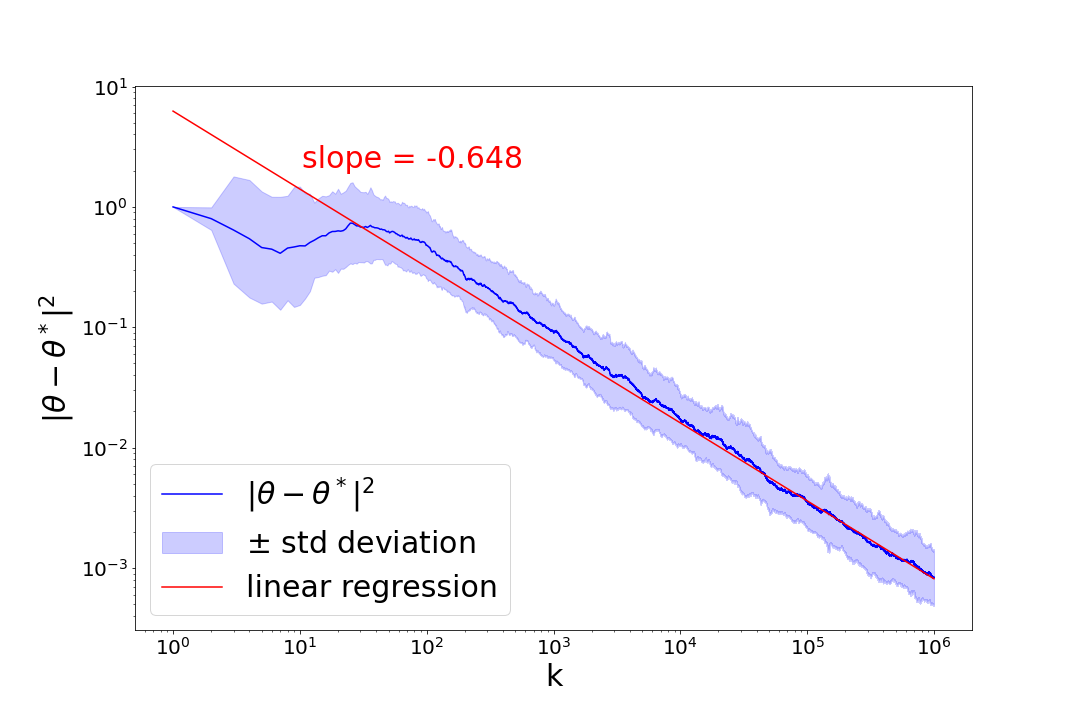}
	\includegraphics[width=0.49\linewidth]{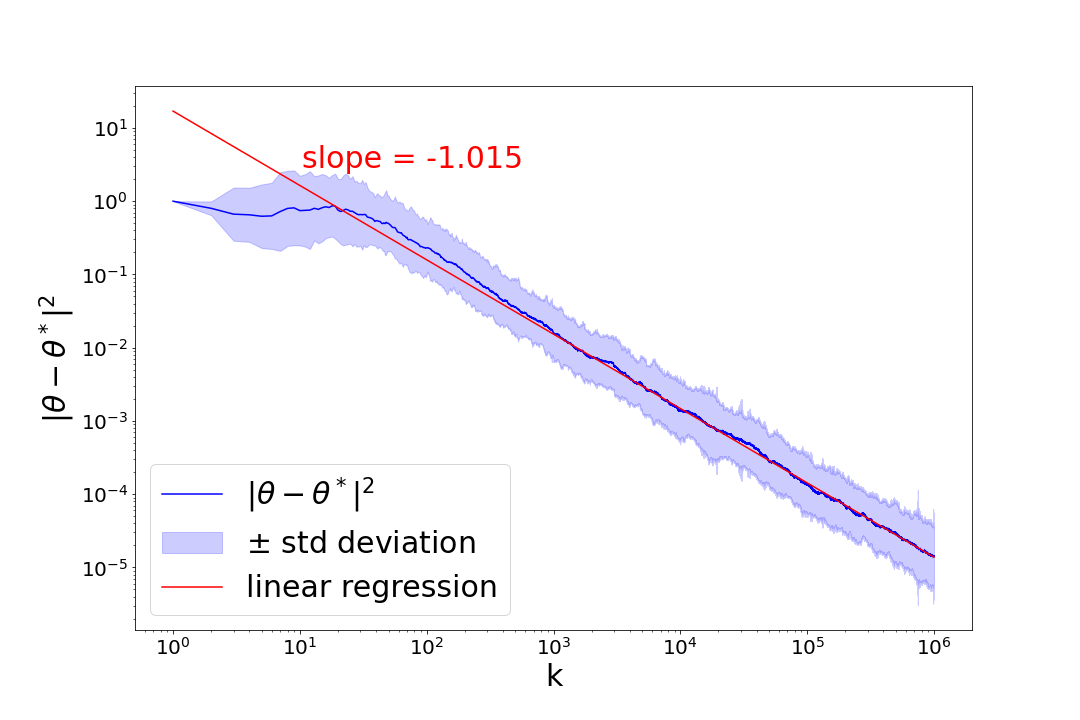}
    \caption{\label{fig:cvg_TD0}
    Empirical errors (averaged over $100$ runs)
    for standard TD(0) (left)
    and stochastic TD(0) (right) under the
    assumptions of Theorem \ref{thm:simple_cvg}.
    The differences with the rates predicted by the theory
    are lower than $3\%$.
    See Appendix \ref{subsec:num_model}
    for details on the model used for the simulations.}
\end{figure}

\begin{remark}
    \label{rk:simple_cvg}
    The fact that standard TD(0) may converge
    is surprising at first sight
    since it may be interpreted as:
    the solution of the BSDE \eqref{eq:BSDE_Y}
    can be learnt without considering
    the stochastic term $Z_t\cdot dW_t$.
    However, there are restrictions
    on the sequence of time steps
    to obtain such a convergence
    results: 
    $\aa_k/\Delta t_k$
    has to tend to zero.
    Such a restriction is totally unnecessary
    for stochastic TD(0) for which
    having both $\Delta t_k$ and
    $\aa_k$ tend to zero is sufficient
    (and $\sum\aa_k$ being divergent).
    We refer to the informal discussion 
    on Appendix \ref{subsec:informal_TD0}
    for some insights on the convergence of standard TD(0)
    and the assumption
    ``$\aa_k/\Delta t_k\to0$''.
    In particular, this assumption is sharp
    since $v(\cdot,\th_k)$ cannot converge to
    a good estimator of $V$ if
    ``$\aa_k/\Delta t_k$'' does
    not tend to zero, 
    see Appendix \ref{subsec:informal_TD0} as well.
\end{remark}

\subsection{Averaging Stochastic TD(0)}
\label{subsec:averaged_cvg}
In this section,
in the same spirit as the results in
\cite{bach2013non} for SGD,
we get the convergence of the stochastic TD(0)
algorithm with an averaging method,
with a constant learning step,
without a strong convexity assumption,
without a regularisation assumption
and without a projection map.
This result cannot be extended
to standard TD(0)
because $\aa/\Delta t_k$ cannot
converge to zero with $\aa>0$ independent of $k$
(see Remark \ref{rk:simple_cvg} for more details).
The averaging method we are using here
to accelerate the convergence
is the Polyak-Juditsky method
\cite{Polyak1992AccelerationOS}
using $\tho_k$ defined by,
for $k\geq1$,
\begin{equation}
    \label{eq:averaging}
    \tho_k
    =
    \frac1k\sum_{i=0}^{k-1}\tht_i, 
\end{equation}
We obtain a convergence rate which is competitive
with the state of the art for the simpler problem
of linear regression using SGD methods
(we cannot expect to beat SGD convergence rates
since SGD can be viewed as a particularly simple
subcase of TD(0),
see Appendix \ref{subsec:SGD} for more details).
\begin{theorem}
    \label{thm:averaged_cvg}
    If $\sum_{i=0}^{\infty}\Delta t_i^2$ is finite,
    there exist $C,R>0$ such that,
    for
    $\aa<R^{-2}$, $k\geq1$,
    \begin{align*}
        \ell(v(\cdot,\bar{\th}_k),v(\cdot,\th^*))
        &\leq
        \frac{C}{\aa k}
        +\frac{C(d+\tr(HH^{-\top}))}{k}.
    \end{align*}
    where $H=\EE\lb\vp(X)\cL\vp(X)^{\top}\rb$
    and $\ell$ is defined in Proposition \ref{prop:langevin}.
    If $\sum_{i=0}^{k-1}\Delta t_i^2\leq a\ln(1+k)$
    for some $a>0$
    and any $k\geq0$,
    then for $\ee>0$,
    there exists $C,R>0$ such that
    for $\aa<R^{-2}$, $k\geq1$,
    \begin{align*}
        \ell(v(\cdot,\bar{\th}_k),v(\cdot,\th^*))
        &\leq
        \frac{C}{\aa k}
        +\frac{C(d+\tr(HH^{-\top}))}{k^{1-\ee}}.
    \end{align*}
\end{theorem}
The proof is adapted from \cite{bach2013non}
with the extra difficulties
that the linear operators applied to $\th_k$
in \eqref{eq:TD0}
are different for each $k\geq0$,
that they are not symmetric,
and that their symmetric part has no
interesting property (only the symmetric part
of the expectation
of its limit when $k\to\infty$ has useful properties).
Moreover, our sequence of stochastic estimators
have vanishing biases that introduce new terms in the proof,
this leads to the necessity to add a
growth assumption on  $\sum_i\Delta t_i^2$.

\begin{figure}
	\centering
    \noindent
	\includegraphics[width=0.49\linewidth]{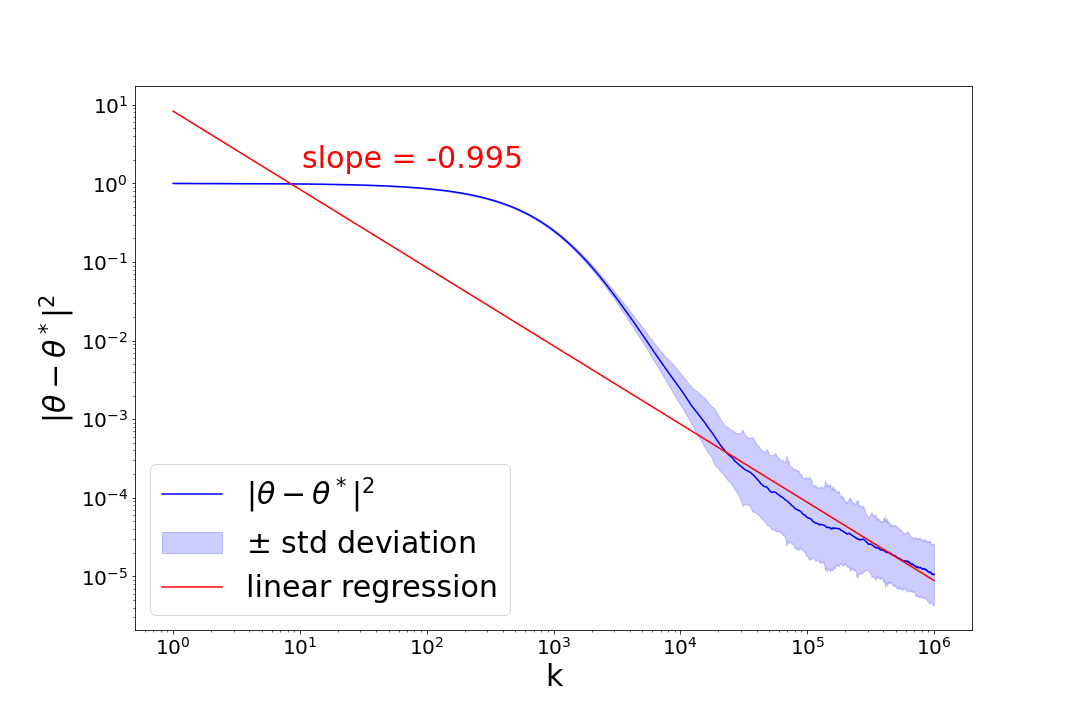}
	\includegraphics[width=0.49\linewidth]{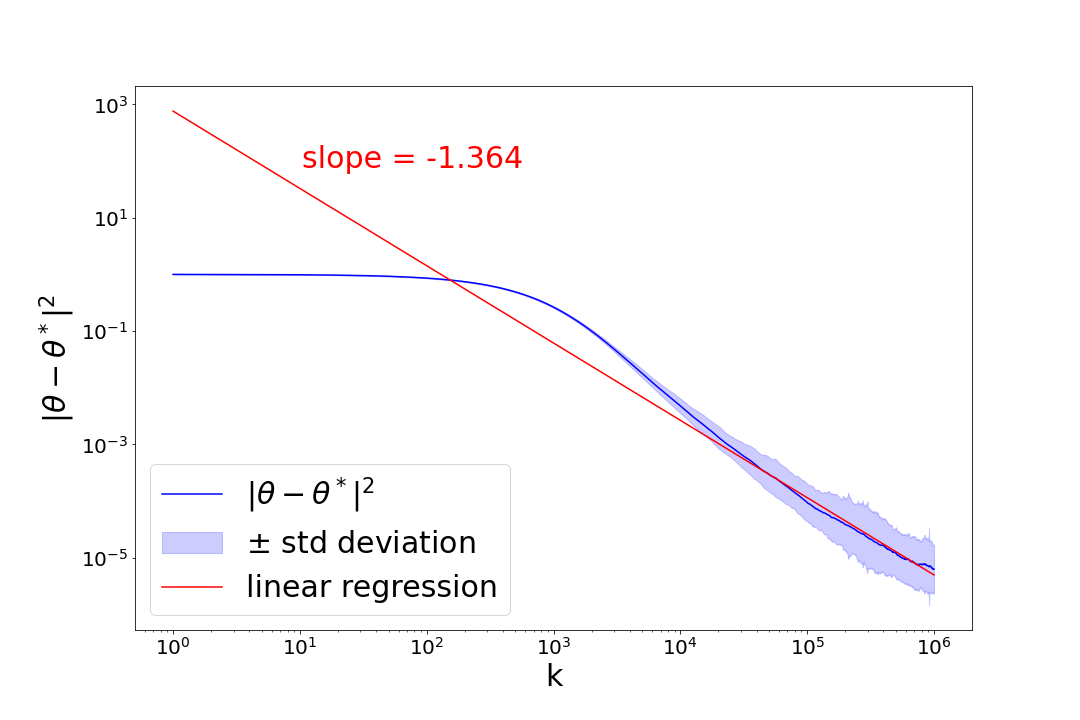}
    \caption{\label{fig:cvg_polyak}
    Empirical errors (averaged over $100$ runs)
    for averaged stochastic TD(0)
    under the
    assumptions of Theorem \ref{thm:averaged_cvg}.
    The learning step is $\alpha=10^{-3}$.
    On the left, we use
    $\Delta t_k=(k+1)^{-\frac12}$
    and obtain the rate predicted by Theorem \ref{thm:averaged_cvg}
    with an error of $0.5\%$.
    On the right, we use 
    $\Delta t_k=(k+1)^{-1}$ and obtain a better rate:
    we believe that this
    is due to $H$
    being symmetric positive definite and
    well-conditioned here but not 
    in Theorem~\ref{thm:averaged_cvg}.
    See Appendix \ref{subsec:num_model}
    for details on the simulations.}
\end{figure}

\section{Extensions and limitations} 
\subsection{Extensions}
\label{subsec:extensions}
\paragraph{The time-dependent case.}
In the present analysis,
we only considered the stationary case,
see \eqref{eq:def_V},
mainly because this is the more standard
setting in RL.
However, we argue that all our results
straightforwardly extend to the time-dependent case,
i.e., when $V$ is given by
\begin{equation*}
        V(t_0,x)
        =
        \EE\lb
        \int_{t_0}^{T}
        r(t,X_t)\,dt
        \Big|X_0=x
        \rb
        \quad\text{ with }\quad
        dX_t
        =
        b(t,X_t)dt
        +\sigma(t,X_t) dW_t.
\end{equation*}
In this case, an observation 
is a quintuple $(t_k,\Delta t_k,X_k,X'_k,R_k)$
where $t_k$ is the time at which $X_k$ is observed.
It allows to solve BSDEs with a new approach.
Indeed, BSDE solvers from the literature
generally divide into two groups.
The first one consists of methods which are
nonlocal with respect to time, i.e.,
the updates use entire trajectories
from $0$ to $T$, e.g.,
\cite{han2016deep,gobet2005sensitivity,
    lehalle1998piecewise,psaltis1988multilayered}.
The methods from the second group
solve a sequence of regression problems
starting from the terminal time $T$
and using a backward induction 
on a fixed time grid, e.g.,
\cite{bachouch2022deep,germain2021neural}.
Our method enter none of these groups
since it is local in time,
it is mesh-free
and it do not rely on backward
inductions.

\paragraph{Nonlinear problems.}
Recall that our main motivation is precisely to
deal with the nonlinear counterparts
of \eqref{eq:def_V}, \eqref{eq:PDE_V}
and \eqref{eq:BSDE_Y}.
We refer to Section \ref{subsec:motivations}
for more details.

\paragraph{Other RL methods.}
We believe that the main ideas of the present
work may be largely repeated
for most of the RL methods using temporal
differences or some of their variants.
For instance,
we refer to Appendix \ref{sec:RG}
for such an extension to residual gradient
\cite{baird1995residual}.
We argue that they may be extended to
methods based on the $Q$-function as well
(which include the majority of the state-of-the-art
RL methods), up to an additional effort since
the limit of the $Q$-function is independent
of its second argument when the time steps
tend to zero, we refer to \cite{tallec2019making}
for a way to overcome this difficulty.

\subsection{Limitations}
\label{subsec:limitations}
\paragraph{About the independence
assumption.}
For observations
from a simulator,
independent samples are
easy to obtain.
For real-world observations,
the independence assumption
is false in general
but almost independent samples
are usually obtained by
shuffling a large dataset:
we believe that it is sufficient
in practice to obtain similar
convergence rates as in the independent situation.
Alternatively,
one might assume
that the observations
are not independent but
are obtained from
following the Markov Chain
(which is rarely done in practice since
datasets are usually shuffled before use).
Theoretically, this alternative assumption
is generally restrictive
since it relies on ergodic limits
and very small learning rates,
making the number of necessary samples
increases dramatically in practice.
We believe that comparable results
under this assumption can be derived as well,
but it is out of the scope of the present paper.
We refer to \cite{TD0RKHS} for convergence rates
under this assumption 
and a comparison with the independence setting,
in the case of a non-parametric TD(0) method.

\paragraph{Few words about the noises.}
In the present work,
we make the assumption
that the noise has independent increments:
this is a limitation
which is almost always assumed in the RL literature,
even in the discrete-time setting.
We make a second assumption, that the noise is continuous
with respect to the time: this is a limitation
since jumps are not allowed, but it
still covers a very large number
of applications from the real world
(e.g., any physical or chemical system).
Given these two assumptions,
at the limit when $\Delta t$ tends to zero,
noises have to be Gaussian by Donsker's theorem
(intuitively, there is a \emph{Gaussianisation}
phenomenon of the noise when $\Delta t$
tends to zero).
For observations from a simulator,
we assumed the noises to be Gaussian
only because it is the simplest way to
have consistent noises
with respect to changes of time steps
(since the sum of independent Gaussians
is Gaussian).
However, our analysis straightforwardly extends
to any family of noises that is
(or at least becomes at the limit $\Delta t\to0$)
consistent with respect to changes of time steps.

\section{Related works}
\label{sec:biblio}

\paragraph{TD learning and SGD.}
The TD algorithm was introduced in the tabular
case by~\cite{sutton1988learning},
with later convergence results for linearly dependent features~\cite{dayan1992convergence}.
Asymptotic stochastic approximation results
were derived by~\cite{jaakkola1993convergence} for the tabular case,
and by~\cite{schapire1996worst} when using linear approximations, with a non-asymptotic analysis in the
\textit{i.i.d.}~sampling case~\cite{lakshminarayanan2018linear}.
The analysis of TD requires tools
from stochastic approximation~\cite{benveniste2012adaptive},
which have mainly been derived 
for stochastic gradient descent (SGD)~\cite{bottou2018optimization} and reused here.
The convergence results  presented in the present
paper may be compared
to standard results on RL
algorithms,
see \cite{bellman1966dynamic,kirk1970optimal}
for TD($0$).
The proof of Theorem \ref{thm:averaged_cvg}
is adapted from the literature on SGD
\cite{Polyak1992AccelerationOS,bach2013non}
to the non-symmetric setting induced by TD(0).
In particular,
\cite{bach2013non} states the state-of-the-art results
concerning convergence of SGD methods
in the non-strongly convex setting,
here we reach similar convergence rates
on the more difficult optimisation
problem raised by TD(0).

\paragraph{Continuous-time RL.}
Continuous-time RL 
started with \cite{baird1993advantage},
who proposed a continuous-time
counterpart to $Q$-learning;
it was later extended by
\cite{tallec2019making}.
From a different perspective,
\cite{bradtke1994reinforcement}
extended classical RL algorithms
to continuous-time discrete-state
Markov decision processes.
Then,
using deterministic dynamics
given by ordinary differential equations,
and based on the Hamilton-Jacobi-Bellman
equation,
\cite{doya2000reinforcement}
derived algorithms for both
policy evaluation
and policy improvement.
Similar deterministic approaches
of continuous-time RL have recently been
explored by
\cite{lutter2021value,yildiz2021continuous}.
In order to balance between exploration
and exploitation,
\cite{wang2020reinforcement}
added an entropy-regularisation term
to a continuous optimisation problem,
the authors concluded that Gaussian controls
are optimal for their relaxed problems,
leading to a similar SDE system as the one
studied in the present work.

\paragraph{Learning methods for solving PDEs and SDEs.}
Solving partial differential equations
using learning algorithms is a natural idea.
Indeed, classical methods 
such as finite differences,
finite elements, or Galerkin methods
cannot be computed for dimensions higher than three
because of the size of the grid becoming too large.
Some mesh-dependent learning algorithms have been developed,
see \cite{lagaris1998artificial,lagaris2000neural,malek2006numerical},
but they suffer from the same computational difficulties
in high dimensions as the classical methods.
There has been a surge of works during the last five years
for solving high-dimensional PDEs or SDEs using deep learning,
let us cite \cite{khoo2021solving,sirignano2018dgm} for the
\emph{Deep Galerkin Method},
or \cite{beck2019machine,han2017deep,han2018solving}
for methods based on
BSDEs or FBSDEs;
we refer to the surveys
\cite{beck2020overview,germain2021neural,hu2023recent} and the references therein 
for more results on deep learning
methods for PDEs, BSDEs or FBSDEs.
The methods presented in the present work
are different to all the above-mentioned
methods since
they are mesh-free, local in time
and do not rely on backward inductions.

\section{Conclusion}
\label{sec:conclusion}
In the present work,
we prove that standard reinforcement
learning methods based on the temporal difference
are not adapted to solve stochastic continuous-time
optimisation problems (see Proposition~\ref{prop:varred}),
nor their discretisations using small time steps.
We then propose two original numerical methods
based on the well-known TD(0) algorithm
using vanishing time steps.
We prove theoretical convergence rates
(Theorem \ref{thm:simple_cvg} and \ref{thm:averaged_cvg})
under the assumption of linear parametrisations
(which is always needed in the literature for proving
explicit rates).
We subsequently verify those rates through numerical
simulations (Figures \ref{fig:cvg_TD0} and \ref{fig:cvg_polyak}).
The first method is model-free,
its convergence rate is slower than
the standard rates for discrete-time settings
because it uses a quantity which is not suited
to small time steps (see Section \ref{sec:TDs}).
Moreover, additional care has to be taken
concerning the choice of the sequence of time steps
(see Remark \ref{rk:simple_cvg}).
The second method is model-based
(it only requires to know the drift function $b$),
it admits better rates of convergence and is more
robust to changes of the time steps.
More precisely,
Theorem \ref{thm:simple_cvg}
shows convergence rates for both algorithms
using standard decreasing
learning rates,
a strong convexity assumption (induced by
the regularisation parameter $\mu$)
and a projection step.
Theorem \ref{thm:averaged_cvg}
shows a fast-convergence rate
for the model-based method
with a constant learning rate,
without strong-convexity assumption
and using an averaging method.
This rate is similar to the state
of the art for the simpler
linear regression problem with SGD
\cite{bach2013non}.

Alternatively, the two methods analysed in the present
work are novel approaches to numerically
solve linear PDEs and BSDEs using observations;
their main advantage on the existing methods
are that they are local in time,
mesh-free and that they do not rely on any backward induction.
Those problems (linear or nonlinear)
are central in the vast 
domain of mathematical modelling
and have uncountable applications,
we refer to Section \ref{subsec:motivations}
for a more thorough discussion and
some examples.

\paragraph{Acknowledgements.}
We thank Justin Carpentier for
fruitful discussions related to this work.
This work was funded in part by the French
government under management of Agence Nationale
de la Recherche as part of the “Investissements d’avenir” program,
reference ANR-19-P3IA-0001(PRAIRIE 3IA Institute).
We also acknowledge support from the European Research Council (grant SEQUOIA 724063).

\bibliographystyle{plain}
\bibliography{biblioML}

\newpage
\appendix
\section{Some standard notations from PDE literature}
\label{sec:app_notations}
In this section, we recall the definition
of some standard differential operators.
For $n,m\geq1$,
let $f:\RR^n\to\RR^m$
be a function which admits
partial derivatives in any direction
up to order two,
for $x\in\RR^n$ we define:
\begin{itemize}
    \item
        the first-order derivative (or Jacobian) of $f$ at $x$
        as $D_xf(x)\in\RR^{m\times n}$
        such that $D_xf(x)_{i,j}=\partial_{x_j}f_i(x)$;
    \item
        if $m=1$,
        the gradient of $f$ at $x$
        as $\nabla_xf(x)\in\RR^{n}$
        such that $\nabla_xf(x)_{j}=\partial_{x_j}f(x)$;
    \item
        if $n=m$,
        the divergence of $f$ at $x$
        as $\divo(f)(x)=\sum_{j=1}^n\partial_{x_j}f_j(x)$;
    \item
        if $m=1$,
        the second order derivative (or Hessian)
        of $f$ at $x$
        as $D^2_xf(x)\in\RR^{n\times n}$
        such that $D^2_xf(x)_{i,j}=\partial_{x_i}\partial_{x_j}f(x)$;
    \item
        if $m=1$,
        the Laplacian of $f$ at $x$,
        as $\Delta_x f(x)=\sum_{i=1}^n\partial_{x_i}\partial_{x_i}f(x)$.
\end{itemize}
Occasionally, the Hessian and the Laplacian might be used
even if $m>1$.
For the Laplacian,
it only consists in applying the Laplacian coordinate-wise.
For the Hessian, it outputs a tensor in dimension three,
such that $D^2_xf(x)_{i,j,k} = 
        \partial_{x_j}\partial_{x_k}f_i(x)$.


\section{Numerical simulations}

\subsection{Description of the model}
\label{subsec:num_model}
Let us describe all the details 
of the numerical simulations leading
to Figure \ref{fig:cvg_TD0} and \ref{fig:cvg_polyak}.
In dimension $d=1$, the one-dimensional
torus is thought as the set $[-0.5,0.5]$
completed with periodic boundary conditions.
We take
$\ss$ and $b$ as in Proposition
\ref{prop:langevin},
and $U$ and $r$ given by
\begin{equation*}
    U
    =
    -\textstyle{\frac{\ss^2}2}
    \ln(2-\cos(2\pi x))
    \;\text{ and }\;
    r(x)
    =
    \lp
    \rho
    +\frac{4\pi^2\ss^2}{2-\cos(2\pi x)}\rp
    \sin(2\pi x).
\end{equation*}
We refer to the top of
Figure \ref{fig:MC} for a graphical representations
of $r$ and $-U$.
The function $-U$ has to be thought of as
a potential function,
i.e., the dynamic
heads toward the minimum of $-U$.
\begin{figure}
	\centering
	\includegraphics[width=\linewidth]{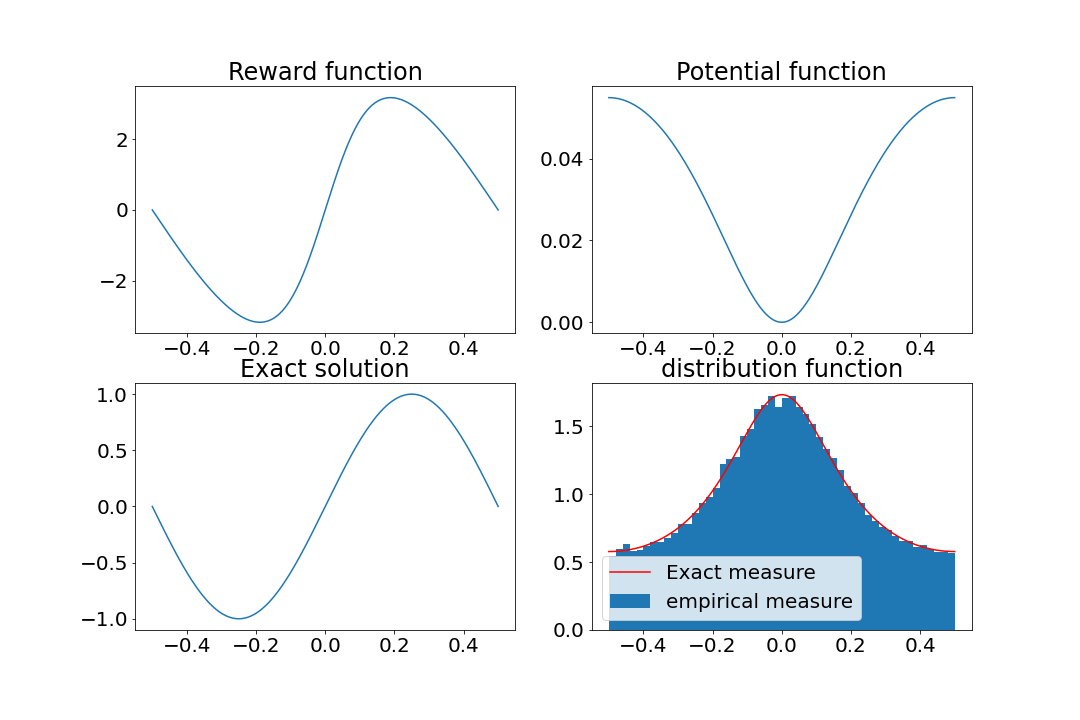}
    \caption{\label{fig:MC}
        Here are some graphical representations
        of the model described in Section \ref{subsec:num_model}.
        Namely,
        the reward function is at the top left;
        the potential is at the top right;
        the exact solution $V$ is at the bottom left;
        the measure $m$ and an empirical approximation
        of it with $10^5$ samples 
        divided in
        $100$ subintervals
        are at the bottom right.
        We used $\rho=1$ and $\ss^2=0.1$.
    }
\end{figure}

One can easily check that $V$ is then given by
\begin{equation*}
    V(x)
    =
    \sin(2\pi x),
\end{equation*}
by checking that \eqref{eq:PDE_V} is satisfied.
The invariant measure $m$ of the dynamics
in \eqref{eq:def_V} satisfies
\begin{equation*}
    m(x)
    =
    \frac{e^{\frac{2U(x)}{\ss^2}}}
    {\int_{\Omega}e^{\frac{2U(y)}{\ss^2}}dy}
    =
    \frac{\sqrt{3}}{2-\cos(2\pi x)}.
\end{equation*}
The function
$V$ and $m$ are represented
at the bottom of Figure \ref{fig:MC}.

The repartition function $F$ of the distribution $m$ satisfies
\begin{equation*}
    F(x)
    =
    m\big([-0.5,x]\big)
    =
    \pi\arctan(\sqrt{3}\tan(\pi x)).
\end{equation*}
Observe that, if $u$ is a uniform
random variable on $[0,1]$
then $F^{-1}(u)$ is
distributed according to $m$.
Therefore, it becomes easy to sample
i.i.d. observations from $m$ given that, for $z\in[0,1]$,
\begin{equation*}
    F^{-1}(z)
    =
    \pi^{-1}\arctan\lp\textstyle{\frac1{\sqrt3}\tan(\pi(z-\frac12))}\rp.
\end{equation*}
The feature vector $\vp$ for the learning is
\begin{equation*}
    \vp(x)
    =
    \big(1,\sin(2\pi x),\cos(2\pi x)\big)^{\top},
\end{equation*}
so that $V$ belongs to the set
of parametrised functions.
The parameters for the learning are
$\aa_k=\frac2{k+1}$,
$\Delta t_k=\sqrt[3]{\aa_k}$
for standard TD(0)
and
$\Delta t_k=\sqrt{\aa_k}$
for stochastic TD(0).
We take $\mu=0$ (i.e., no regularisation)
but we obtain the convergence rates
of Theorem \eqref{thm:simple_cvg}
like if $\mu$ was equal to $0.5$
because the matrix
$H=\EE\lb\vp(X)\cL\vp(X)^{\top}\rb$
is positive symmetric definite with
$H\geq 0.5 I_d$.

\newpage

\section{Insights and informal discussions}
\label{sec:informal}
\subsection{Insights on the convergence of standard TD(0)}
\label{subsec:informal_TD0}
As we already pointed out
in Section \ref{subsec:contributions}
and Remark \ref{rk:simple_cvg},
the very fact that the model-free TD(0) method
may converge is already surprising
as it may be interpreted as:
        the solution of \eqref{eq:BSDE_Y} can be learnt
        via an iterative method which totally omits
        the stochastic term ``$Z_t\cdot dW_t$''.

The present section aims
at giving insights on
why this method should converge
in a particular setting,
i.e.,
when $\alpha_k/\Delta t_k$ converges to zero.
Conversely, we will see that
it cannot converge outside of the latter
setting.
All the arguments here are very informal
and should not be considered as rigorous
(we refer to
Appendix~\ref{proof:cumrew} for the rigorous proof).

Let us replace the discrete iteration
variable $k$ by a continuous variable $s\in[0,\infty)$
with the standard analogous 
\begin{equation*}
    s
    \approx
    \sum_{i=0}^{k-1}
    \aa_i
    \;\text{ and }\;
    \Delta s
    \approx
    \alpha_k.
\end{equation*}
We assume that all indexes $k$
may be changed into the continuous
variable $s$
and that we can give a sense,
at least informally,
to continuous-iteration observations
$(\Delta t_s,X_s,X'_s,r_s)_{s\geq0}$.
We get,
\begin{align*}
        d\th_s
        \approx
        \Delta \th_s
        &\approx
        \th_{s+\Delta s}-\th_s
        \\
        &\approx
        -\Delta s
        \,\delta_s
        \nabla_{\th}v
        \\
        &\approx
        -\Delta s\,
        (\cL v - r_s
        +(\Delta t_s)^{-\frac12}(\ss^{\top}\nabla_xv)\cdot\xi_s)
        \nabla_{\th}v
        \\
        &\approx
        -\Delta s\,(\cL v - r_s)
        \nabla_{\th}v
        +\sqrt{\textstyle{\frac{\Delta s}{\Delta t_s}}}
        (\ss^{\top}\nabla_xv)\cdot(\sqrt{\Delta s}\,\xi_s)\nabla_{\th}v
    \end{align*}
where the third line is obtained using
Lemma \ref{lem:expand_eps},
with $\xi_s$ a normally distributed
random variable.

Then, let us use
a new continuous-iteration Brownian motion
$(\Wt_s)_{s\geq0}$
(which is naturally different from
$(W_t)_{t\geq0}$ the continuous-time
Brownian motion from the dynamics
in \eqref{eq:def_V})
and its standard discretised approximation
$d\Wt_s\approx \sqrt{\Delta s}\,\xi_s$.
We get that $(\th_s)_{s\geq0}$
is a solution to the SDE
\begin{equation}
    \label{eq:SDE_th}
        d\th_s
        =
        -(\cL v -r_s)\nabla_{\th}v\, ds
        +\text{\guillemotleft}
            \textstyle{\sqrt{\frac{\aa_k}{\Delta t_k}}}
            \text{\guillemotright}
        \nabla_\th v\nabla_xv^{\top}\ss\,d\Wt_s,
\end{equation}
The first consequence of the latter calculus
is that $\th$ cannot converge
to a deterministic limit
if the diffusion term
        $\text{\guillemotleft}
            \textstyle{\sqrt{\frac{\aa_k}{\Delta t_k}}}
            \text{\guillemotright}
        \nabla_\th v\nabla_xv^{\top}\ss$
        does not converge to zero.
This should only happen if
$\textstyle{\sqrt{\frac{\aa_k}{\Delta t_k}}}$
tends to zero.
Observe that we did not use
any assumption on the parametrisation
$v(\cdot,\th)$.
In particular,
the informal derivation
of \eqref{eq:SDE_th}
holds for nonlinear parametrisation,
so does the latter conclusion.

From now on, we explicitly
assume that $v$ is linear in $\th$,
i.e., $v(\cdot,\th)=\th^{\top}\vp$
and that
$\textstyle{\sqrt{\frac{\aa_k}{\Delta t_k}}}$
does tend to zero.
Consequently, the dynamics of $\th_s$
tends to be deterministic
at the limit and we get
    \begin{align*}
        \frac{d}{ds}\EE\lb|\th_s-\th^*|^2\rb
        &=
        -2(\th-\th^*)^{\top}\EE_{X\sim m}\lb\vp(\cL v(X,\th)-r(X))\rb
        \\
        &=
        -2(\th-\th^*)^{\top}
        \EE\lb\vp\cL \vp^{\top}\rb
        (\th-\th^*),
    \end{align*}
where the definition of $\th^*$ in \eqref{eq:def_th*}
is used to get the last line.
Recall that the symmetric part of
$\EE\lb\vp\cL \vp^{\top}\rb$
is positive definite by Lemma \ref{lem:HSA}.
We just proved that        
$\EE\lb|\th_s-\th^*|^2\rb$
is a Lyapunov function,
which implies that $\th_s$
should converge to $\th^*$.

\subsection{Comparison between TD(0) and SGD: the limit $\rho\rightarrow\infty$}
\label{subsec:SGD}
We already mentioned that
TD(0) may be seen as an extension
of SGD.
A way to make the latter statement
rigorous here is to consider
the limit when $\rho$ tends to infinity.
Indeed,
from \eqref{eq:def_V}
we easily obtain that,
for $x\in\Omega$,
\begin{equation*}
    \lim_{\rho\to\infty} \rho V(x)
    =
    r(x),
\end{equation*}
and that $\ell/\rho$
converges to the square loss,
where $\ell$ is defined in Proposition
\ref{prop:langevin}.
In addition,
if $\rho\Delta t$ tends to zero,
from a fixed
starting state $X=x\in\Omega$,
we get
\begin{equation*}
    \ddt_{\Delta t}
    =
    \rho v(x,\th)
    -r(x)
    +o(1).
\end{equation*}
Formally,
we can thus view
the least-square regression of $r$
using SGD
as a particularly simple limit case of TD(0).

In the discrete-time setting,
the latter analysis is even easier
and only consists in taking
the discrete discount factor
(named $\gamma$ in \eqref{eq:def_dd}) equal to zero.
In this situation, the fact that SGD
is an instance of TD(0) in a particularly simple
setting (some difficult terms
are removed, such that the
one making TD(0) be a bootstrapping
algorithm)
appears more clearly.

This explains why we do not
expect the convergence rates
proved in the present work to 
beat the state of the art
of the literature on SGD.
Let us recall that
Theorem \ref{thm:averaged_cvg}
does yield a similar convergence
rate as the state-of-the-art
results on mere convex problem
with averaging method
\cite{bach2013non}.

\subsection{Convergence of the Discrete Markov Chain}
\label{subsec:discrete_MC}
In this section,
we give a condition under which
the convergence results \eqref{eq:cvg_mk}
holds.
Here, we consider the case
of observations obtained from a
simulator, like stated in Section
\ref{sec:assumptions}.
Moreover,
we assume that $m_k$, the law of
$X_k$, is in fact the
stationary measure of the discretisation
of the dynamic in \eqref{eq:def_V}
by the Euler-Maruyama scheme.
In this case,
the sequence $(m_k)_{k\geq0}$ is weakly convergent
to $m$ and a convergence rate is given in the following theorem.
\begin{theorem}[Theorem $14.5.1$
    from \cite{MR1214374}]
    \label{thm:cvg_mk}
    For $f\in C^4(\Omega;\RR)$,
    there exists $C>0$ depending
    only on the $C^4$-norm of $f$ such that
    $\labs\EE\lb f(X_k)-f(X)\rb\rabs\leq C\Delta t_k$.
\end{theorem}

\subsection{Formal discussion on the origins of the noises}
\label{subsec:noises}
In standard discrete-time reinforcement
learning, two kinds of noises are often considered.
The first one comes from the transition probability
of the MDP, this is the intrinsic noise of the model.
The second kind of noise is an artificial noise
which is generally added to favor exploration.
Exploration is used in order to exit
non-interesting local minima and converge
to more robust solutions.

Here, the noise in \eqref{eq:def_V}
is represented by a Brownian motion,
but we never explained if this noise
was intrinsic to the model or artificially added
for exploration.
In this section, we show that it can be any or both
of the two propositions.
More precisely, we introduce
the following three classes of models:
\begin{enumerate}[label=\bf{M\arabic*}]
    \item
        \label{class:int}
        stochastic models of the form of
        \eqref{eq:def_V} with an intrinsic noise,
    \item
        \label{class:art}
        deterministic models with linear dynamics
        with respect to the controls,
        and an artificial noise added for exploration,
        regularisation
        or for smoothing the control 
        (see \cite{lidec2022leveraging}).
    \item
        \label{class:both}
        stochastic models of the form of \eqref{eq:def_V},
        with linear dynamics with respect to the controls,
        and an artificial noise. 
\end{enumerate}
For the first class of models \ref{class:int},
the noise is part of the model and cannot be tuned.
The last two classes seem more interesting
in the framework of RL, and more specifically
in the theoretical study of the exploration/exploitation
trade-off.
In class \ref{class:art},
the dynamics has the following form,
\begin{equation*}
    \label{eq:exploration}
    \frac{d}{dt}x_t
    =
    A(x_{t})u(x_{t})
    +B(x_{t}),
\end{equation*}
where $u$ is the control function,
$A$ and $B$ are respectively matrix-valued and vector-valued functions.
Then, in order to encourage exploration,
instead of choosing a deterministic control function
(e.g., being greedy with respect to some criterion),
one generally adds noises
in the choice of $u$.
Gaussian noises are often considered in discrete dynamics
because of their
simplicity to sample,
or because they are the minimisers
of some entropy-relaxations of the optimisation problems
(see \cite{wang2020reinforcement} for instance).
Therefore, at least at the discrete level,
it is natural to change $u$
into its noisy counterpart $u+\ss(x,u)\xi_i/\sqrt{\Delta t}$.
This leads to the following dynamics,
\begin{equation*}
    X_{t_{i+1}}
    =
    X_{t_{i}}
    +\Delta t\lp A(X_{t_i})u(X_{t_i})
    +B(X_{t_i})\rp
    +\sqrt{\Delta t}A(X_{t_i})\ss(X_{t_i},u(X_{t_i}))\xi_i,
\end{equation*}
which admits a similar form as
the observations from a simulator
in Section \ref{sec:assumptions},
with $A\ss$ replacing $\ss$.
This time, the noise
is tunable and a particular
interesting regime consists in letting
$\ss$ tends to zero.
The class \ref{class:both}
consists in a mix between the two other classes,
with $A\ss_{\text{art}}+\ss_{\text{int}}$
replacing $\ss$ this time,
where $\ss_{\text{art}}$ and $\ss_{\text{int}}$ are
the artificial and intrinsic noises respectively.
The noise is tunable in some measure but the regime
$\ss\to0$ is in general prohibited.

\section{Links optimal control problems}
\label{sec:appli_RL}
\subsection{A short review of the optimal control problem in continuous time}
\label{subsec:review_OC}
Let us consider the controlled counterpart of 
\eqref{eq:def_V},
\begin{equation}
    \label{eq:controlled_SDEX}
    dX_t
    =
    b(X_t,u(X_t))dt
    +\ss(X_t,u(X_t)) dW_t,
\end{equation}
where $u:\Omega\to \cA$ is a control function
and $\cA$ is set of admissible controls.
The controller aims at maximising
the following quantity
over the set of admissible functions $u$,
\begin{equation}
    \label{eq:def_J}
    J(u)
    =
    \EE\Big[
    \int_0^{\infty} e^{-\rho t} r(X_t,u(X_t))dt\Big],
\end{equation}
where $X_0$ is distributed according to some
probability measure $\mu_0\in\cP(\Omega)$.
Contrary to the ones in
the main text,
the prototypes of the drift, diffusion and reward
functions are given by
$b:\Omega\times\cA\to\RR^d$,
$\ss:\Omega\times\cA\to\RR^d$
and 
$r:\Omega\times\cA\to\RR$,
i.e., they may depend on the control.

A natural approach from optimal control
theory is to introduce the value
$V^u$ associated to a specific control function
$u$, defined by:
\begin{equation*}
    \label{eq:def_Vu}
    V^u(x)
    =
    \EE\Big[
    \int_0^{\infty} e^{-\rho t} r(X_t,u(X_t))dt\Big|X_0=x\Big].
\end{equation*}
Then, solving the optimisation problem \eqref{eq:def_J}
boils down to compute $V^*$ and $u^*$,
respectively
the optimal value function
and the optimal control,
satisfying
\begin{equation*}
    V^*(x)
    =
    \max_u V^u(x)
    \qquad
    \text{ and }
    \qquad
    u^*(x)
    \in
    \argmax_u
    V^u(x).
\end{equation*}
Moreover, under mild regularity
assumption, $V^*$
can be characterised as the
solution of a partial differential equation,
named Hamilton-Jacobi-Bellman (HJB) equation
(see \cite{bellman1966dynamic}
for more details),
given by
\begin{equation}
    \label{eq:def_general_HJB}
    \rho V^*
    -
    \max_{u\in \cA}
    H(x,\nabla_xV^*(x),D^2_{x}V^*(x),u)
    =
    0,
\end{equation}
where the Hamiltonian
$H$ is defined by,
for $p\in\RR^d$, $z\in\RR^{d\times d}$
and $u\in \cA$,
\begin{equation}
    \label{eq:def_H}
    H(x,p,z,u)
    =
    r(x,u)
    +p\cdot b(x,u)
    +\frac12\tr\lp\lp\ss\ss^{\top}\rp(x,u)z\rp.
\end{equation}
Moreover, the optimal
control belongs to the $\argmax$
the Hamiltonian,
i.e.,
\begin{equation*}
    u^*(x)
    \in
    \argmax_u
    H(x,\nabla_xV^*(x),D^2_{x}V^*(x),u).
\end{equation*}

\paragraph{Uncontrolled diffusion.}
When the diffusion $\ss$ does not depend
on the control,
The HJB equation \eqref{eq:def_general_HJB}
can be rewritten as
\begin{equation*}
    \rho V^*
    -\tr\lp\textstyle{\frac{\ss\ss^{\top}}{2}}D^2_{x,x}V\rp
    -\Ht(x,\nabla_xV^*)
    =
    0,
\end{equation*}
where, here, for $p\in\RR^d$,
the reduced Hamiltonian
is defined by,
\begin{equation*}
    \Ht(x,p)
    =
    \max_{u\in\RR^d}
    \lc p\cdot b(x,u) + r(x,u)\rc.
\end{equation*}
In this case and under
standard assumptions
on $b$ and $L$,
the BSDE \eqref{eq:BSDE_Y}
is replaced by the FBSDE system
\begin{equation*}
    \left\{
    \begin{aligned}
        dX_t
        &=
        \nabla_p\Ht(X_t,Z_t)\,dt
        +\ss(X_t)dW_t,
        \\
        dY_t
        &=
        -(r(X_t)
        -\rho Y_t)\,dt
        +Z_t^{\top}\ss(X_t)\,dW_t.
    \end{aligned}
    \right.
\end{equation*}
Unlike the couple of SDEs from
\eqref{eq:def_V} and
\eqref{eq:BSDE_Y},
the latter system is coupled
through its both equations,
making it much more complex to solve.
Observe that it cannot be solved
neither forward nor backward,
it generally requires a fix point
argument on the whole system.

\paragraph{A simple example.}
Let us consider the simpler case
where $\ss$ is a constant positive
real number, $b$ is given
by $b(x,u)=u$,
$r$ is concave with respect to $u$,
and the control space
is $\cA=\RR^d$.
In this case,
the HJB equation \eqref{eq:def_general_HJB}
can be rewritten as
\begin{equation*}
    \label{eq:simple_HJB}
    \rho V^*
    -{\textstyle\frac{\ss^2}2}
    \Delta_x V^*
    -\Ht(x,\nabla_xV^*)
    =
    0,
\end{equation*}
where, here, for $p\in\RR^d$,
the reduced Hamiltonian
is defined by,
\begin{equation*}
    \Ht(x,p)
    =
    \max_{u\in\RR^d}
    \lc p\cdot u + r(x,u)\rc.
\end{equation*}
In particular, $\Ht$
is the Legendre's transform
(or convex conjugate)
of $-r$ with respect to
its second argument,
let us recall that $-r$
is assumed to be convex with respect to $u$.
In this case,
the optimal control $u^*$ admits the following
closed form,
\begin{equation}
    \label{eq:opt_cont}
    u^*(x)
    =
    \nabla_p \Ht(x,\nabla_xV^*(x)).
\end{equation}
\begin{remark}
    If $\Ht$ and $\nabla_p\Ht$
    admit a known closed form,
    the latter example is 
    particularly simple.
    This conclusions easily extend 
    to the case of $b$ being
    a more general affine function
    with respect to $u$,
    i.e., of the following form,
    \begin{equation*}
        b(x,u)
        =
        b_0(x)
        +b_1(x)u,
    \end{equation*}
    where $b_0:\Omega\to\RR^d$
    and $b_1:\Omega\to\RR^{d\times d}$ are
    vector-valued and matrix-valued functions
    respectively.

    Hopefully, this class of control
    problems is in fact of high interest since
    it actually contains a lot of models from
    modern control theory
    and reinforcement learning.
    One may for instance consider
    the cases where the dependence of $r$
    with respect to $u$ is:
    either a characteristic function
    of a compact subset of $\RR^d$;
    or a power function of $|u|$
    (in the next paragraph, we present
    the quadratic case).
\end{remark}

\paragraph{The quadratic case.}
Under the same assumptions as in the latter
example, let us consider the particular case
when $r$ is separated with a quadratic part in $u$.
More precisely, let us consider $r$ to be given by,
\begin{equation*}
    r(x,u)
    =
    -\frac{|u|^2}2
    +r_0(x),
\end{equation*}
where $r_0:\Omega\to\RR$
is the state reward function.
In this case, the conditions
on $V^*$ and $u^*$
may be written as follows,
\begin{align*}
    &\rho V^*(x)
    -{\textstyle\frac{\ss^2}2}
    \Delta_x V^*(x)
    -{\textstyle\frac{1}2}
    |\nabla_xV^*(x)|^2
    =
    r_0(x),
    \\
    &u^*(x)
    =
    \nabla_xV^*(x).
\end{align*}
Finally, let us mention that the
latter problem is not
strictly speaking a linear-quadratic problem
even if $b$ is linear
and $r$ is quadratic
with respect to $u$.
Indeed, linear-quadratic control problems
requires $r$ to be quadratic
with respect to the couple $(x,u)$
(and $b$ linear with respect to $(x,u)$),
which is not the case here.

\subsection{Informal extension to solve control problems}
\label{subsec:PI}
Let us assume that
the observations are from a simulation,
as defined in Section \ref{sec:assumptions}.
The analysis in the present section
can then easily be repeated for
real-world observations
with different notations.
Using
the Euler-Maruyama discretisation
scheme on \eqref{eq:controlled_SDEX}
we obtain the
controlled discrete step operator
\begin{equation*}
    \label{eq:controlled_SDEX_dis}
    X_{t_{i+1}}
    =
    S_{\Delta t}(X_{t_i},u(X_{t_i}),\xi_i)
    :=
    X_{t_{i}}
    +\Delta t\, b(X_{t_i},u(X_{t_i}))
    +\sqrt{\Delta t}\ss(X_{t_i},u(X_{t_i}))\xi_i.
\end{equation*}
This time, the latter step operator
does not characterise a Markov chain,
but a Markov decision process
associated with the reward function $r$.

In reinforcement learning,
such MDP are usually solved using 
iterative algorithms in the class of
\emph{generalised policy iteration} (GPI)
methods \cite{MR3889951}.
Those algorithms
are generally based on the
value function
or the $Q$-function.
They consist in alternating
two interactive updates:
the \emph{policy evaluation} 
and the \emph{policy improvement}.
Here, using the framework introduced
in the present work,
the policy evaluation consists in
computing an approximation
of the value function $V^u$
associated to some control function
$u$.
Conversely,
the policy improvement consists
in updating the control
function $u$ in order
to maximise its associated
value function.
These two processes
are therefore antagonist
in the sense that
updating $u$ makes our current
approximation of $V^u$
being less accurate;
and vice-versa, when
$V^u$ is updated, the optimal response
to it changes as well.

In the present work,
outside of the present section,
we focus on the policy evaluation process.
However, one may easily figure out
how it may be extended to GPI methods
in order to solve MDPs with vanishing time steps.
The simplest example consist in the
\emph{policy iteration} method
which is described below.

\paragraph{Theoretical policy iteration.}
It consists in computing
an approximation of the value function
between any update of the control function.
Namely, starting from an initial arbitrary
control function $u^0$,
we compute the approximating
sequences
$(V^{\ell})_{\ell\geq0}$ and
$(u^{\ell})_{\ell\geq0}$
as follows:
\begin{equation*}
    V^{\ell}
    =
    V^{u^{\ell}}
    \qquad
    \text{ and }
    \qquad
    u^{\ell+1}
    \in
    \argmax_u H(x,\nabla_xV^{\ell}(x),D^2_{x}V^{\ell}(x),u),
\end{equation*}
where $V^{u^{\ell}}$ is the value function
associated to the control function $u^{\ell}$,
and $u^{\ell+1}$ is the best response
given the value function~$V^{\ell}$.
Let us recall that $H$ is defined
in \eqref{eq:def_H} as the Hamiltonian.

This algorithm is convergent with a super-linear
convergence rate, see~\cite{puterman1979convergence,
    puterman1981convergence,kerimkulov2020exponential}.
In particular, this method
may be seen as
a Newton algorithm
applied to some
infinite-dimensional fixed-point operator.

As the terminology \emph{theoretical} suggests,
the latter method is not implementable in practice
with finite computational power,
because of our assumption
on continuous state and control
spaces.
This is not the case of the following iterative
method.

\paragraph{Approximate policy iteration.}
This method is inspired by the latter one,
theoretical policy iteration.
However, this time
the policy-evaluation step
is only made using functional approximation
using one of the two TD(0) methods
(model-based or model-free)
presented in the present work.

We assume that the policy improvement
step can be done (resp. approximatively solved),
for any value function $V:\Omega\to\RR$.
More precisely,
there exists an operator 
$\cU$ taking $V$ as an argument
and outputting a control function $u=\cU(V)$,
such that $u(x)$ is a maximiser (resp. almost a maximiser)
of $u'\mapsto H(x,\nabla_xV(x),D^2_{x}V(x),u')$.
For instance,
$\cU$ may admit a closed form if
the system reduces to \eqref{eq:opt_cont}
as in the example of the previous section.
Otherwise, one may use an approximating iterative
method to construct $\cU$, for instance with
an actor-critic method.

Like in the main text, we consider
a parametrised value function
$x\mapsto v(x,\th)$
for some parameter
$\th\in\Theta$.
Starting from an arbitrary initial parameter
$\th^0$,
let $(\th^{\ell})_{\ell\geq0}$
be a sequence
of parameters
such that $x\mapsto v(x,\th^{\ell})$
is approximating the above sequence
$(V^{\ell})_{\ell\geq0}$
in the theoretical policy iteration method.
The sequence of control functions
is defined by,
\begin{equation*}
    u^{\ell+1}
    =
    \cU(v(\cdot,\th^{\ell})).
\end{equation*}
Then, at iteration $\ell\geq1$,
we compute $\th^{\ell}$
using the stochastic TD(0) method
from
\eqref{eq:TD0}
(alternatively we could have chosen the
standard TD(0) method),
i.e.,
\begin{equation*}
    \th^{\ell}
    =
    \lim_{k\to\infty} \th^{\ell}_k,
    \:\text{ where }
    \th^{\ell}_{k+1}
    =
    \th^{\ell}_k
    -\aa^{\ell}_k
    \ddt^{\ell}_k
    \nabla_{\th}v(X^{\ell}_k,\th^{\ell}_k),
\end{equation*}
using the following 
definitions of
the counterpart of
\eqref{eq:def_ddt},
\begin{align*}
    &\ddt^{\ell}_k
    =
    \ddt_{\Delta t_k}(X^{\ell}_k,\Xt^{\ell}_k,\th^{\ell}_k)
    :=
    (\Delta t^{\ell}_k)^{-1}
    \big(  v(X^{\ell}_k,\th^{\ell}_k)
    - \gg_{\Delta t_k} v(\Xt^{\ell}_k,\th^{\ell}_k)
    - r(X^{\ell}_k,u(X^{\ell}_k))\Delta t^{\ell}_k
    +Z_k
    \big),
    \\
    &\text{where }
    Z^{\ell}_k
    =
    \lp \Xt^{\ell}_k
    -X^{\ell}_k
    -b(X^{\ell}_k,u^k(X^{\ell}_k))\Delta t^{\ell}_k\rp
    \cdot
    \nabla_xv(X^{\ell}_k,\th^{\ell}_k),
    \\
    &\text{and }
    \Xt^{\ell}_k
    =
    S_{\Delta^{\ell}_k}
    (X^{\ell}_k,u^{\ell}(X^{\ell}_k),\xi^{\ell}_k).
\end{align*}
For a fixed $\ell\geq1$,
the sequences
$(\Delta t^{\ell}_k)_{k\geq0}$,
$(X^{\ell}_k)_{k\geq0}$,
$(\xi^{\ell}_k)_{k\geq0}$ and
$(\aa^{\ell}_k)_{k\geq0}$
satisfy similar assumptions
as their counterparts in the main text.

\paragraph{Other RL methods for solving MDPs.}
Like the latter adaptation to continuous time of
the approximate policy iteration method,
most of the RL algorithms using temporal difference
may be adapted using the current framework
in order to be more robust to vanishing time steps.
This includes in particular approximate
value iteration, Q-learning, SARSA, actor-critic methods
and others\dots.
The changes only consists in replacing any
temporal difference in an algorithm by
one of the two TD(0) algorithms presented here.

\begin{remark}
    For the model-based algorithm,
adding the variance-reduction
correction can only benefit to
the policy evaluation process.
This explains why we chose
to focus only on policy evaluation
algorithms like TD learning
in the present work.
Another reason for not considering
the policy improvement process
is that the necessary assumptions
for making its analysis
are different from the ones considered here.
Therefore, we believe that
dealing with the policy improvement process
in continuous time in a separate future contribution
will allow
a better understanding of each work and more flexibility
to extend our results.
\end{remark}

\section{The residual gradient method}
\label{sec:RG}
\subsection{Extensions of the main results to the residual gradient method}
\label{subsec:RG}
In this section, 
we state similar results
as Theorems
\ref{thm:simple_cvg}
and \ref{thm:averaged_cvg}
while replacing the TD(0) methods 
by residual gradient (RG) methods.
We want to insist that,
in Section \ref{sec:RG},
the notations
$\th^*$, $\th^*_{\mu}$,
$(\th_k)_{k\geq0}$ and $(\tht_k)_{k\geq0}$
stand for different quantities
than in the rest of the article.
This is due to the iterations
and limits
from RG methods being different from
the ones from TD(0).
Let us start by defining those
quantities here.
The standard and stochastic RG methods
write as
\begin{equation}
    \label{eq:RG}
    \tag{RG}
    \begin{aligned}
        \th_{k+1}
        &=
        \th_k
        -\aa_k\dd_{k}\nabla_{\th}\dd_k,
        \\
        \tht_{k+1}
        &=
        \tht_k
        -\aa_k\ddt_{k}\nabla_{\th}\ddt_k,
    \end{aligned}
\end{equation}
where $\aa_k$ is the learning rate.
We also introduce the regularised
standard and stochastic 
RG methods as
\begin{equation}
    \label{eq:muRG}
    \tag{$\mu$-RG}
    \begin{aligned}
        \th_{k+1}
        &=
        \Pi_{B_M}\lp
        \th_k
        -\textstyle{\frac{\aa_k}2}\nabla_{\th}
        (|\dd_k|^2
        +\mu|\th_k|^2)\rp,
        \\
        \tht_{k+1}
        &=
        \Pi_{B_M}\lp
        \tht_k
        -\textstyle{\frac{\aa_k}2}\nabla_{\th}
        (|\ddt_k|^2
        +\mu|\tht_k|^2)\rp,
    \end{aligned}
\end{equation}
where $\mu\geq0$ and 
$M$ is assumed to be a known
upper bound of $\th^*_{\mu}$
which is defined later as
the limit of the $\mu$-regularised method.
Let us define $F_{\mu}$ the RG cost,
\begin{equation*}
    F_{\mu}(\th)
    =
    \underbrace{\EE_X\lb(\cL v(X,\th)-r(X))^2\rb}_{\text{interesting term}}
    + \underbrace{\textstyle{\frac{1}2}\EE_X\lb\tr\lp(\ss\ss^{\top} D^2_xv(X,\th))^2\rp\rb}_{\text{bias from RG being a gradient method}}
    +\underbrace{\textstyle{\frac{\mu}2}|\th|^2}_{\text{bias from the regularisation}},
\end{equation*}
Observe that the third term in the definition of $F_{\mu}$
is a similar bias as
the one we have for TD(0),
it can be tuned in a similar
fashion as in Corollary \ref{cor:Non_Asymp}.
However, the second term is a bias coming from 
RG being a gradient method
(more details in Section \ref{sec:TDs})
and is more problematic since it cannot
be reduced, but if $\ss$ can be tuned.

We then define $\th^*_{\mu}$ as the eventual limit of RG
(which admits a double bias),
\begin{equation*}
    \th^*_{\mu}
    =
    \argmin_{\th\in\Theta}F_{\mu}(\th).
\end{equation*}
We define $F=F_0$
and $\th^*=\th^*_{0}$.
The following theorem is
the counterpart to RG of Theorem \ref{thm:simple_cvg},
it concerns the convergence rate of the regularised
stochastic RG method.
\begin{theorem}
    \label{thm:simple_cvg_RG}
    Assume
    \ref{hypo:v_lin},
    \ref{hypo:rb_reg},
    \ref{hypo:vp_reg},
    $\mu>0$,
    $\aa_k=\frac{2}{\mu(k+1)}$
    and
    $\Delta t_k\leq c/\sqrt{k+1}$,
    for some $c>0$ and
    for any $k\geq0$.
    The sequence
    $(\tht_k)_{k\geq0}$ is convergent,
    and there exists $C>0$
    such that, for $k\geq1$,
    \begin{equation*}
        \EE\lb\labs\tht_k-\th^*_{\mu}\rabs^2\rb
        \leq
        \frac{C}{\mu^2k}.
    \end{equation*}
\end{theorem}
We refer to Section \ref{proof:simple_cvg_RG}
for the proof.

Unlike Theorem \ref{thm:simple_cvg},
the latter theorem does
not stand any convergence result
for the standard RG method.
This is because standard RG method
may only converge to a constant
function and not to a good estimator of $V$,
more details are given in Section
\ref{sec:TDs}.

Then, we state below 
the counterpart to RG of Theorem \ref{thm:averaged_cvg},
it concerns the convergence rate of the unregularised
RG method with constant learning step and an
averaging method.
\begin{theorem}
    \label{thm:averaged_cvg_RG}
    Assume
    \ref{hypo:v_lin},
    \ref{hypo:rb_reg}
    and
    \ref{hypo:vp_reg}.
    and that $\th^*$ is bounded.
    If $\sum_{i=0}^{\infty}\Delta t_i^2$ is finite,
    there exist $C,R>0$ such that,
    the following inequality holds for
    $\aa<R^{-2}$, $k\geq1$,
    \begin{align*}
        \EE\lb |\cL v(X,\bar{\th}_k)
        -\cL v(X,\th^*)|^2\rb
        &\leq
        \frac{C}{\aa k}
        +\frac{Cd}{k},
    \end{align*}
    where
    $\bar{\th}_k=\frac1k\sum_{i=0}^{k-1}\tht_i$, for $k\geq1$.

    If instead we assume that
    $\sum_{i=0}^{k-1}\Delta t_i^2\leq a\ln(1+k)$
    for some $a>0$
    for any $k\geq0$,
    then for any $\ee>0$
    there exists $C,R>0$ such that
    for $\aa<R^{-2}$, $k\geq0$,
    the latter inequalities are replaced
    with the following ones
    respectively
    \begin{align*}
        \EE\lb |\cL v(X,\bar{\th}_k)
        -\cL v(X,\th^*)|^2\rb
        &\leq
        \frac{C}{\aa k}
        +\frac{Cd}{k^{1-\ee}}.
    \end{align*}
\end{theorem}
Observe that the left-hand sides in the latter
theorem are less convenient to work
with than the ones in Theorem \ref{thm:averaged_cvg};
this is another advantage in favor of
the semi-gradient method TD(0)
over real gradient methods like RG.

The proof of Theorem \ref{thm:averaged_cvg_RG}
is a simple adaptation of the ones
of Theorem \ref{thm:averaged_cvg}, it is even
simpler
since here all the linear operators 
are symmetric.

\subsection{Possible extensions of RG}
\label{subsec:alternatives_RG}
One important weakness of RG algorithm
is the presence of
the term
$\EE_X\lb\tr\lp(\ss\ss^{\top} D^2_xv(X,\th))^2\rp\rb$
in the definition of $F_{\mu}$.
In this section, we propose four alternatives
to remove it.
More precisely,
we will try to minimise
$\Ft_{\mu}$
instead of
$F_{\mu}$,
where 
$\Ft_{\mu}$ is
defined by
\begin{equation*}
    \Ft_{\mu}(\th)
    =
    \EE_X\lb(\cL v(X,\th)-r(X))^2\rb
    +\frac{\mu}2|\th|^2,
\end{equation*}
and $\th^*_{\mu}$
is now defined by 
$\th^*_{\mu}=\argmin \Ft^*_{\mu}$.
Similarly,
we take $\Ft=\Ft_0$
and $\th^*=\th^*_0$.

In particular, Theorems \ref{thm:simple_cvg_RG}
and \ref{thm:averaged_cvg_RG} hold with the following four
alternatives.
\\[.35cm]
\textbf{Multi-step RG.}
This algorithm consists in the following induction relation,
\begin{equation}
    \label{eq:MSRG}
    \tag{MS-RG}
    \tht_{k+1}
    =
    \tht_k
    -\textstyle{\frac{\aa_k}2}\nabla_{\th}
    \lp|\ddo_k|^2
    +\mu|\tht_k|^2\rp,
    \;
    \text{ with }
    \ddo_k
    =
    \frac1{n_k}\sum_{i=0}^{n_k-1}
    \ddt_{\Delta t_k}(X_{k,t_i},X_{k,t_{i+1}},\th_k),
\end{equation}
where
$X_{k,0}=X_k$
and
    $X_{k,t_{i+1}}
    =
    S_{\Delta t_k}(X_{k,t_i},\xi_{k,i})$,
for $0\leq i<n_k$
and $n_k\geq1$ a sequence converging to infinity.
The conclusions of Theorem \ref{thm:simple_cvg_RG} then hold
with
        $\aa_k=\frac{4}{\mu(k+1)}$,
        $n_k\geq c^{-1}\sqrt{k+1}$,
        $n_k\Delta t_k\leq c/\sqrt{k+1}$,
        and $\ss_k\leq ck^{-\frac18}$,
        for some $c>0$,
        for any $k\geq0$.
In particular the proofs or the counterparts
to the multistep setting
of Theorems 
\ref{thm:simple_cvg_RG}
and \ref{thm:averaged_cvg_RG},
are similar to the originals
but using
Lemma \ref{lem:expand_eps_n_sq} below
instead of Lemma \ref{lem:expand_eps}.
\\[.35cm]
\textbf{Vanishing viscosities.}
This algorithm consists in the following induction relation,
\begin{equation}
    \label{eq:ssRG}
    \tag{$\ss$RG}
    \tht_{k+1}
    =
    \tht_k
    -\textstyle{\frac{\aa_k}2}\nabla_{\th}
    \lp|\ddt_k|^2
    +\mu|\tht_k|^2\rp,
    \;
    \text{ with }
    \ddt_k
    =
    \ddt_{\ss_k,\Delta t_k}(X_k,\xi_k,\th_k),
\end{equation}
where $\ddt_{\ss_k,\Delta t_k}$ 
is $\ddt_{\Delta t_k}$
where $\ss$ has been replaced by $\ss_k$
in
the dynamics
and
\eqref{eq:def_ddt}.
Here, we assume that we may choose the intensity of
the noise, which is only possible when the noise
have been added artificially to a deterministic problem
(see class \ref{class:art}),
which is in general interesting for the three
following reasons: allowing exploration,
having regular continuous-time solutions
and having full-supported invariant measures
of the dynamics.

The conclusions of Theorem \ref{thm:simple_cvg_RG} then hold
with
        $\aa_k=\frac{4}{\mu(k+1)}$,
        $\Delta t_k\leq c/\sqrt{k+1}$,
        and $\ss_k\leq ck^{-\frac18}$,
        for $k\geq0$.
\\[.35cm]
\textbf{Using mini-batches.}
Another alternative consists in
using mini-batches,
i.e.,
\begin{equation}
    \label{eq:MBRG}
    \tag{MB-RG}
    \th_{k+1}
    =
    \th_k
    -\textstyle{\frac{\aa_k}2}\nabla_{\th}
    \lp|\ddo_k|^2
    +\mu|\tht_k|^2\rp,
    \;
    \text{ with }
    \ddo_k
    = 
    \frac1{N_k}
    \sum_{i=1}^{N_k}
    \ddt_{\Delta t_k}(X_k,\xi_{k,i},\th_k),
\end{equation}
where $(N_k)_{k\geq0}$ are the size of the mini-batches.
The conclusions of Theorem \ref{thm:simple_cvg_RG} then hold
with
        $\aa_k=\frac{4}{\mu(k+1)}$,
        $\Delta t_k\leq c/\sqrt{k+1}$,
        and $N_k\geq c^{-1}\sqrt{k}$,
        for $k\geq0$.
\\[.35cm]
\textbf{Changing the law of the noise.}
Note that the perturbating term
from \eqref{eq:mean_sq_TD}
comes from the variance of 
a term involving
$\xi_k^{\top}D^2_xv\xi_k-\Delta_xv$.
Let us make the simple observation
that, in dimension $d=1$,
the latter expression is null
if $\xi_k$ is a Rademacher random variable.
This argument can  be generalised
to dimension $d>1$.
Since $D^2v(X_k,\th_k)$ is symmetric,
we can find $P$ an orthogonal matrix and
$D$ a diagonal matrix such that
$D^2_xv(X_k,\th_k)=P^{\top}DP$.
Therefore, it we can take
$\xi_k = P^{\top}\zeta_k$
where $\zeta_k$ is a random vector,
each of its coordinate being an
independent Rademacher random variable.

Using Donsker's theorem
\cite{donsker1951invariance},
the random process at the limit is still
a Brownian motion even if the increments
before convergence are not Gaussian anymore.
However, the weak convergence of the sequence
$(m_k)_{k\geq0}$ is slower here:
$\Delta t_k$ is replaced by $\Delta t_k^{\frac12}$
(this is a consequence of
the central limit theorem).
The conclusions of Theorem \ref{thm:simple_cvg_RG} then hold
with
        $\aa_k=\frac{4}{\mu(k+1)}$
        and
        $\Delta t_k\leq c/(k+1)$,
        for $k\geq0$.

    \section{Proof of the main results}
    \label{proof:cumrew}

        Here, $C$ is a constant that can change
        from line to line and is independent from
        $(\aa_k)_{k\geq0}$, $(\th_k)_{k\geq0}$ and $\mu$.

        We will also make the simplification assumption
        that $\ss$ is always a constant positive number,
        only to simplifies the notations.
        No additional difficulties (like regularity)
        come from such a simplification
        since the learnt function $v$ is assumed to be regular.

    \subsection{Proof of Theorem \ref{thm:simple_cvg}}
    \label{proof:simple_cvg}
    
    In order to prove Theorem \ref{thm:simple_cvg},
    we will need the following theorem first.
    \begin{theorem}
        \label{thm:TD0_stoch}
        Let $A\in\RR^{d\times d}$ be a square matrix
        such that $A+A^{\top}\geq 2\mu I_d$ for some $\mu>0$,
        and $b,\th^*\in\RR^d$ such that
        $A\th^*=b$ and $|\th^*|\leq M$ for some $M\geq0$.
        For $\th_0\in\Theta$,
        define the sequence $(\th_k)_{k\geq0}$
        by induction as
        \begin{equation*}
            \th_{k+1}
            =
            \Pi_{B(0,M)}\lp\th_k-\aa_kg_k\rp,
        \end{equation*}
        for $k\geq0$,
        where 
        $\aa_k>0$
        is convergent to zero and $\sum_{k\geq0}\aa_k=\infty$,
        $\labs\EE\lb g_k|\th_k\rb-A\th_k+b\rabs \leq (1+|\th_k|)\ee_k$,
        $\ee_k\geq0$ is convergent to zero,
        and $\EE\lb|g_k|^2\big|\th_k\rb\leq c_k(1+|\th_k|^2)$
        for some $c_k\geq0$.
        Then $(\th_k)_{k\geq0}$ is convergent in expectation
        to $\th^*$ and
        \begin{equation*}
            \EE\lb|\th_k-\th^*|^2\rb
            \leq
            4M^2e^{-\mu\sum_{i=0}^{k-1}\aa_i}
            +C(1+M^2)
            \sum_{i=0}^{k-1}
            \aa_i\lp\aa_i+\mu^{-1}\ee_i^2\rp
            e^{-\mu\sum_{j=i+1}^{k-1}\aa_j}.
        \end{equation*}

    \end{theorem}
    \begin{proof}
        Up to starting the
        iterative algorithm from $\th_1$ instead of $\th_0$,
        we may assume that $|\th_k|\leq M$ for every $k\geq0$.
        For $k\geq0$,
        let us denote $y_k=\EE\lb|\th_k-\th^*|^2\rb$.
        We recall that
        $|\Pi_{B(0,M)}(\th)-\th^*|\leq|\th-\th^*|$
        for any $\th\in\Theta$,
        since $\th^*\in B(0,M)$.
        This and the induction relation satisfied
        by $\th_k$, imply
        \begin{align*}
            y_{k+1}
            &=
            \EE\lb\labs\Pi_B\lp\th_k-\aa_kg_k\rp-\th^*\rabs^2\rb
            \\
            &\leq
            \EE\lb\labs\th_k-\th^*-\aa_kg_k\rabs^2\rb
            \\
            &\leq
            y_k
            -2\aa_k\EE\lb(\th_k-\th^*)^{\top}g_k\rb
            +\aa_k^2\EE\lb \labs g_k\rabs^2\rb
            \\
            &\leq
            y_k
            -2\aa_k\EE\lb(\th_k-\th^*)^{\top}\EE\lb g_k|\th_k\rb\rb
            +\aa_k^2\EE\lb\EE\lb \labs g_k\rabs^2|\th_k\rb\rb
            \\
            &\leq
            y_k
            -2\aa_k\EE\lb(\th_k-\th^*)^{\top}(A\th_k-b)\rb
            +2\aa_k\ee_k\EE\lb|\th_k-\th^*|(1+|\th_k|)\rb
            +c_k\aa_k^2\EE\lb(1+|\th_k|^2)\rb
            \\
            &\leq
            y_k
            -\aa_k\EE\lb(\th_k-\th^*)^{\top}(A+A^{\top})(\th_k-\th^*)\rb
            +\mu\aa_k\EE\lb|\th_k-\th^*|^2\rb
            +2(1+M^2)\mu^{-1}\aa_k\ee_k^2
            +c_k(1+M^2)\aa_k^2
            \\
            &\leq
            (1-\mu\aa_k)y_k
            +(1+M^2)\aa_k\lp C\mu^{-1}\ee_k^2+c_k\aa_k\rp
            \\
            &\leq
            e^{-\mu\aa_k}y_k
            +(1+M^2)\aa_k\lp C\mu^{-1}\ee_k^2+c_k\aa_k\rp.
        \end{align*}
        where we used a Young inequality to get to the fifth line.
        Therefore, we obtain,
        \begin{equation*}
            y_k
            \leq
            e^{-\mu\sum_{i=0}^{k-1}\aa_i}
            y_0
            +C(1+M^2)
            \sum_{i=0}^{k-1}
            \aa_i\lp c_i\aa_i+\mu^{-1}\ee_i^2\rp
            e^{-\mu\sum_{j=i+1}^{k-1}\aa_j},
        \end{equation*}
        which leads to the desired inequality
        using $y_0\leq (|\th_0|+|\th^*|)^2\leq4M^2$.
    \end{proof}

    \begin{proof}[Proof of Theorem \ref{thm:simple_cvg}]
        The proof only consists in checking that we can
        apply Theorem \ref{thm:TD0_stoch}.
        Let us start by the proof for stochastic TD(0).
        Using the same notation as in Theorem
        \ref{thm:TD0_stoch}, we define,
        \begin{equation*}
            A
            =
            \EE\lb\vp(X)\cL(X)\rb+\mu I_d,
            \quad
            b
            =
            \EE\lb r(X)\vp(X)\rb,
            \;\text{ and }
            g_k
            =
            \ddt_k\vp(X_k)
            +\mu\th_k.
        \end{equation*}
        Then, we get
        \begin{align*}
            \EE\lb g_k|\th_k\rb
            &=
            \EE\lb\vp(X_k)\lp \cL\vp(X_k)
            +R_{0,k}
            +\Delta t_k^{\frac12}R_{1,k}
            +\Delta t_kR_{2,k}\rp\rb\th_k
            +\mu\th_k
            +\EE\lb\vp(X_k)r(X_k)\rb
            \\
            &=
            \EE\lb\vp(X_k)\cL\vp(X_k)\rb\th_k
            +\mu\th_k
            +\EE\lb\vp(X_k)r(X_k)\rb
            +\Delta t_k\EE\lb \vp(X_k) R_{2,k}^{\top}\rb\th_k,
        \end{align*}
        where
        $R_{0,k}=R_0(X_k,\xi_k)$,
        $R_{1,k}=R_1(X_k,\xi_k)$ and
        $R_{2,k}=R_2(\Delta t_k,X_k,\xi_k)$
        are given in Lemma \ref{lem:expand_eps}.
        From \eqref{eq:cvg_mk},
        we get
        \begin{equation*}
            \labs\EE\lb\vp(X_k)\cL\vp(X_k)\rb
            -A\rabs
            \leq
            C,
            \;\text{ and }
            \labs\EE\lb\vp(X_k)r(X_k)\rb -b\rabs
            \leq
            C.
        \end{equation*}
        Therefore, we obtain
        $\labs\EE\lb g_k|\th_k\rb-A\th_k-b\rabs
        \leq C(1+|\th_k|)\ee_k$ with $\ee_k=\Delta t_k$.
        The fact that $\EE\lb|g_k|^2\big|\th_k\rb\leq c_k(1+|\th_k|^2)$,
        with $c_k=C$ being independent of $k$,
        is straightforward. 
        Finally, $A+A^{\top}\geq2\mu I_d$
        comes from Lemma \ref{lem:HSA}.
        Theorem \ref{thm:TD0_stoch}
        and the inequalities $|\th^*|\leq C\mu^{-1}$
        and $\displaystyle{\exp\lp-\sum_{j=i+1}^k\frac1j\rp}
        \leq i/k$ for $k> i\geq0$,
        conclude the proof for stochastic TD(0).

        For standard TD(0),
        the proof is similar but we have to change $c_k$
        into $c_k=C(\Delta t_k)^{-1}$.
    \end{proof}

\subsection{Proof of Theorem \ref{thm:averaged_cvg}}
    \label{proof:averaged_cvg}

We start with the following definitions,
\begin{align*}
    S
    &=
    \rho\EE\lb
    \vp(X)\vp(X)^{\top}\rb
    +\frac{\ss^2}2\EE\lb
    D_x\vp(X)D_x\vp(X)^{\top}\rb
    \\
    A
    &=
    \EE\lb
    \vp(X)\lp\frac{\ss^2}2\nabla_x\ln(m)
    +b\rp D_x\vp(X)^{\top}\rb
    \\
    H(x)
    &=
    \vp(x)\cL\vp(x)^{\top}
    \\
    H_k(x)
    &=
    H(x)
    +\EE\lb H(X)-H(X_k)\rb
    \\
    H
    &=
    \EE\lb H(X)\rb.
\end{align*}

\begin{proof}[Proof of Theorem \ref{thm:averaged_cvg}]

    Here, $C>0$ stands for a generic constant
    which value may change from line
    to line,
    it depends on the constants in the assumptions
    and is independent of $k$, of the smallest
    eigenvalue of $S$ and of~$\aa$.

    Using Lemma \ref{lem:expand_eps},
    we get
    \begin{equation*}
        \th_{k+1}
        =
        \th_k
        -\aa \vp(X_k)
        \lp
        \cL\vp(X_k)
        +R_0(X_k,\xi_k)
        +\Delta t_k^{\frac12}R_1(X_k,\xi_k)
        +\Delta t_kR_2(\Delta t_k,X_k,\xi_k)
        \rp^{\top}
        \th_k
        +\aa\vp(X_k)r(X_k),
    \end{equation*}
    where $R_0(x,\xi)^{\top}\th
    =\frac{\ss^2}2\lp
    \xi^{\top}D^2_xv(x,\th)\xi
    -\Delta_x v(x,\th)\rp$,
    and $R_1$ and $R_2$
    can be red in Lemma \ref{lem:expand_eps},
    and we get
    $\EE_{\xi}[R_0(x,\xi)]=\EE[R_1(x,\xi)]=0$.
    Take $\eta_k=\th_k-\th^*$,
    it satisfies the following
    induction relation,
    \begin{equation*}
        \eta_{k+1}
        =
        \lp I_d-\aa H_k(X_k)\rp\eta_k
        -\aa\lp H_k(X_k)\th^*+\vp(X_k)r(X_k)\rp
        -\aa\lp H-\EE[H(X_k)]
        +\Delta t_k\vp(X_k)R_{2,k}^{\top}\rp(\eta_k+\th^*),
    \end{equation*}
    where $H_k(x)=\vp(x)(\cL\vp(x)
    +R_{0,k}+\Delta t_k^{\frac12}R_{1,k})^{\top}
    +H-\EE[H(X_k)]$,
    in particular $\EE[H_k(X_k)]=H$.
    One may easily check that $\eta_k$ can be rewritten
    as $\eta_k =\sum_{r=0}^{k-1}\eta_k^r$,
    where $\eta_k^r$ is defined by,
    \begin{equation}
        \label{eq:def_etar}
        \begin{aligned}
        \eta_{k+1}^r
        &=
        (I_d-\aa H)\eta_{k}^r
        +\chi^r_k
        +\Delta t_k\psi^r_k,
        \\
        \eta_0^0
        &=
        \eta_0,
        \quad
        \quad
        \quad
        \eta_0^r
        =
        0
        \text{ if }
        r\geq1,
    \end{aligned}
    \end{equation}
    where $\chi^r_k$ and $\psi^r_k$ are
    defined by
    \begin{equation}
    \begin{aligned}
        \chi_k^0
        &=
        \aa(H-H_k(X_k))\th^*
        +\aa\lp
        \vp(X_k)r(X_k)
        -\EE\lb\vp(X_k)r(X_k)\rb
        \rp,
        \\
        \psi_k^0
        &=
        \aa\Delta t_k^{-1}
        \lp\EE[H(X_k)]-H\rp\th^*
        +\aa\Delta t_k^{-1}
        \EE[\vp(X_k)r(X_k)-\vp(X)r(X)]
        -\aa\vp(X_k)R_{2,k}^{\top}\th^*,
        \\
        \chi_k^{r+1}
        &=
        \aa(H-H_k(X_k))\eta^{r}_k,
        \\
        \psi_k^{r+1}
        &=
        \aa\lp\Delta t_k^{-1}
        (\EE[H(X_k)]-H)
        -\vp(X_k)R_{2,k}^{\top}\rp\eta^{r}_k,
    \end{aligned}
    \end{equation}
    where we used that $\EE\lb\vp(X)\cL v(X,\th^*)\rb=0$
    to get the second line.
    One may notice that $\eta^r_k=0$
    if $r\geq k$.

    \emph{First step: getting bounds on the covariance matrices of $\chi^k_r$ and $\psi^k_r$.}
    Here, we prove by induction on $r$ and $k$ that
    \begin{align*}
        \EE\lb\eta_k^r\otimes\eta_k^r\rb
        &\leq
        3C_k\aa^{r}R^{2r}I_d,
        \\
        \EE\lb\chi_k^r\otimes\chi_k^r\rb
        &\leq
        C_k\aa^{\max(r+1,2)}R^{2r}S,
        \\
        \EE\lb\psi_k^r\otimes\psi_k^r\rb
        &\leq
        \ee
        C_k\aa^{\max(r+1,2)}R^{2r}S,
    \end{align*}
    where
    $R^2=3\Ct\lp
    \norminf{\cL\vp+\EE[ R_0(\cdot,\xi_0)]}
    +\Delta t_0^{\frac12}\norminf{R_1(\cdot,\xi)]}
    +2\ee^{-1}\sup_{k\geq0}
    \norminf{R_2(\Delta t_k,\cdot,\xi)]}+2\ee^{-1}\rp$,
    $0<\ee<\Delta t_0^{-2}$ is a constant that will be defined later,
    $\Ct$ is the constant from Lemma \ref{lem:cov_vp}
    and
    $C_k=\lp|\th^*|^2+\eta_0^{\top}S\eta_0\rp\exp(\ee\sum_{i=0}^{k-1}\Delta t_i^2)$.

    For $k\geq0$, and $r\geq1$,
    let us prove the results for $(k+1,r)$
    while assuming that it
    holds for $(k,r)$, $(k,r-1)$
    and $(k+1,r-1)$.
    For $b_k=\ee\Delta t_k^{2}$,
    we get from \eqref{eq:def_etar}
    and \eqref{eq:bound_I_gH1},
    \begin{align*}
        \EE\lb\eta^r_{k+1}\otimes\eta^r_{k+1}\rb
        &\leq
        \lp1+b_k\rp\EE\lb(I_d-\aa H)\eta^r_k\otimes\eta^r_k
        (I_d-\aa H^{\top})\rb
        +\EE\lb\chi^r_k\otimes\chi^r_k\rb
        +\Delta t_k^2(1+b_k^{-1})\EE\lb\psi^r_k\otimes\psi^r_k\rb
        \\
        &\leq
        3C_k\aa^{r}R^{2r}
        \lp1+b_k\rp(I_d-\aa H)
        (I_d-\aa H^{\top})
        +C_k\aa^{r+1}R^{2r}S
        +\ee C_k\Delta t_k^2\aa^{r+1}R^{2r}(1+b_k^{-1})S
        \\
        &\leq
        3C_k\aa^rR^{2r}
        (1+\ee\Delta t_k^2)(I_d-\aa S)
        +\aa^{r+1}R^{2r}C_k(2+\ee\Delta t_k^2) S
        \\
        &\leq
        3C_k\aa^rR^{2r}(1+\ee\Delta t_k^2)I_d
        \leq
        3C_ke^{\ee\Delta t_k^2}\aa^rR^{2r}I_d
        =
        3C_{k+1}\aa^rR^{2r}I_d.
    \end{align*}
    Then, concerning $\chi^r_{k+1}$,
    using Lemma \ref{lem:cov_vp}, we get
    \begin{align*}
        \EE\lb\chi^r_{k+1}\otimes\chi^r_{k+1}\rb
        &\leq
        3C_{k+1}\aa^{r+1}R^{2r-2}
        \EE\lb (H-H_k(X_k))(H-H_k(X_k))^{\top}\rb
        \\
        &\leq
        3C_{k+1}\aa^{r+1}R^{2r-2}
        \EE\lb H_k(X_k)H_k(X_k)^{\top}\rb
        \\
        &\leq
        3C_{k+1}\aa^{r+1}R^{2r-2}
        \norminf{\cL\vp+\EE[ R_0(\cdot,\xi_k)+\Delta t_k^{\frac12}R_1(\cdot,\xi^k)]}
        \EE\lb \vp(X_k)\otimes\vp(X_k)^{\top}\rb
        \\
        &\leq
        C_{k+1}\aa^{r+1}R^{2r}S.
    \end{align*}
    Finally, using Lemma \ref{lem:cov_vp} once again for $\psi^r_{k+1}$, we get,
    \begin{align*}
        \EE\lb\psi^r_{k+1}\otimes\psi^r_{k+1}\rb
        &\leq
        6C_{k+1}\aa^{r+1}R^{2r-2}
        \lp
        \Delta t_k^{-2}(\EE[H(X_k)]-H)(\EE[H(X_k)]-H)^{\top}
        +\EE\lb|R_{2,k}|^2\vp(X_k)\otimes\vp(X_k)\rb
        \rp
        \\
        &\leq
        \ee C_{k+1}\aa^{r+1}R^{2r}S.
    \end{align*}
    It remains to prove the inequalities
    for $k=0$ and $r=0$.
    Concerning $r=0$, the proof is similar
    but we use the boundedness of $\th^*$ and $r$
    instead of the induction assumption.
    Then $k=0$ and $r\geq1$ is straightforward
    since $\eta^r_0=\chi^r_0=\psi^r_0=0$.
    
    \emph{Second step: getting a bound on $\EE\lb\lp\bar{\eta}_k^r\rp^{\top}S\bar{\eta}_k^r\rb$}.
    Namely, we will prove that
    \begin{equation*}
        \EE\lb\lp\bar{\eta}_k^r\rp^{\top}S\bar{\eta}_k^r\rb
        \leq
        \frac{C\aa^{\max(r-1,0)}R^{2r}}k
        \tr(I_d+H^{-\top}H)
        \lp
        \frac1k\sum_{i=0}^{k-1}C_i
        +\frac1k\lp\sum_{i=0}^{k-1}\Delta t_iC_i^{\frac12}\rp
        +\ddt_{r=0}\aa^{-1}\rp,
    \end{equation*}
    for some constant $C>0$.
    First, we notice that
    \begin{align*}
        \eta_k^r
        &=
        \lp I_d-\aa H\rp^{k-1}\eta_0^r
        +\sum_{i=0}^{k-1}\lp I_d-\aa H\rp^{k-1-i}
        \lp \chi^r_i+\Delta t_i\psi^r_i\rp
        \\
        \bar{\eta}_k^r
        &=
        \frac1{\aa k}H^{-1}\lp I_d-\lp I_d-\aa H\rp^{k}\rp\eta_0^r
        +\frac1{\aa k}\sum_{i=0}^{k-1}
        \lp I_d -\lp I_d-\aa H\rp^{k-i}\rp H^{-1}
        \lp \chi^r_i+\Delta t_i\psi^r_i\rp,
    \end{align*}
    this and the fact that $\EE\lb\chi_i^r\rb=0$ imply 
    \begin{multline*}
        \EE\lb\lp\bar{\eta}_k^r\rp^{\top}S\bar{\eta}_k^r\rb
        \leq
        \frac3{\aa^2k^2}
        (\eta_0^r)^{\top}
        \lp I_d-\lp I_d-\aa H\rp^{k}\rp^{\top}H^{-\top}S
        H^{-1}\lp I_d-\lp I_d-\aa H\rp^{k}\rp\eta_0^r
        \\
        +\frac3{\aa^2k^2}
        \sum_{i=0}^{k-1}
        \EE\lb(\chi_i^r)^{\top}
        \lp I_d -( I_d-\aa H^{\top})^{k-i}\rp
        H^{-\top}SH^{-1}
        \lp I_d -\lp I_d-\aa H\rp^{k-i}\rp(\chi_i)^r\rb
        \\
        +\frac{3}{\aa^2k^2}
        \sum_{0\leq i,j\leq k-1}
        \Delta t_i\Delta t_j
        \EE\lb(\psi_i^r)^{\top}
        \lp I_d -( I_d-\aa H^{\top})^{k-i}\rp
        H^{-\top}SH^{-1}
        \lp I_d -\lp I_d-\aa H\rp^{k-j}\rp\psi_j^r\rb.
    \end{multline*}
    Let us define $I^r_{k,0}$, $I^r_{k,1}$ and $I^r_{k,2}$
    as the first, second  and third term, respectively,
    in the right-hand side of the latter inequality.
    One may notice that $I^r_{k,0}=0$ if $r\geq1$.
    Then concerning, $I^0_{k,0}$, we get
    \begin{align*}
        I^0_{k,0}
        &=
        \frac3{2\aa^2k^2}
        \eta_0^{\top}
        \lp I_d -( I_d-\aa H^{\top})^{k}\rp
        \lp H^{-\top}+H^{-1}\rp
        \lp I_d -\lp I_d-\aa H\rp^{k}\rp\eta_0
        \\
        &\leq
        \frac{C}{\aa^2k}\eta_0^{\top}\eta_0
        \leq
        \frac{C}{\aa^2k},
    \end{align*}
    where we used \eqref{eq:bound_I_gH4} to obtain
    the last line.
    Then let us pass to $I^r_{k,1}$,
    \begin{align*}
        I^r_{k,1}
        &=
        \frac3{2\aa^2k^2}
        \sum_{i=0}^{k-1}
        \EE\lb(\chi^r_i)^{\top}
        \lp I_d -( I_d-\aa H^{\top})^{k-i}\rp
        \lp H^{-\top}+H^{-1}\rp
        \lp I_d -\lp I_d-\aa H\rp^{k-i}\rp\chi^r_i\rb
        \\
        &=
        \frac3{2\aa^2k^2}
        \tr
        \sum_{i=0}^{k-1}
        \lp I_d -( I_d-\aa H)^{k-i}\rp
        \EE\lb\chi^r_i\otimes\chi^r_i\rb
        \lp I_d -( I_d-\aa H^{\top})^{k-i}\rp(H^{-\top}+H^{-1})
        \\
        &\leq
        \frac{3\aa^{\max(r-1,0)}R^{2r}}{2k^2}
        \tr
        \sum_{i=0}^{k-1}
        C_i
        \lp I_d -( I_d-\aa H)^{k-i}\rp S
        \lp I_d -( I_d-\aa H^{\top})^{k-i}\rp(H^{-\top}+H^{-1})
        \\
        &=
        \frac{3\aa^{\max(r-1,0)}R^{2r}}{4k^2}
        \tr
        \sum_{i=0}^{k-1}
        C_i
        \lp I_d -( I_d-\aa H^{\top})^{k-i}\rp
        \lp I_d -( I_d-\aa H)^{k-i}\rp(2I_d+HH^{-\top}+H^{-1}H^{\top})
        \\
        &\leq
        \frac{C\aa^{\max(r-1,0)}R^{2r}}{k^2}
        \tr(I_d+HH^{-\top})
        \sum_{i=0}^{k-1}C_i.
    \end{align*}
    Then, concerning $I^r_{k,2}$, using the triangular inequality,
    we get
    \begin{align*}
        I^r_{k,2}
        &\leq
        \frac{3}{2\aa^2k^2}
        \lp
        \sum_{i=0}^{k-1}
        \Delta t_i
        \EE\lb(\psi^r_i)^{\top}
        \lp I_d -( I_d-\aa H^{\top})^{k-i}\rp
        (H^{-\top}+H^{-1})
        \lp I_d -\lp I_d-\aa H\rp^{k-i}\rp\psi^r_i\rb^{\frac12}\rp^2
        \\
        &\leq
        \frac{C\aa^{\max(r-1,0)}R^{2r}}{2k^2}
        \lp\sum_{i=0}^{k-1}
        \Delta t_i
        \lb C_i
        \tr\lp I_d+H^{-\top}H\rp\rb^{\frac12}\rp^2
        \\
        &=
        \frac{C\aa^{\max(r-1,0)}R^{2r}}{2k^2}
        \tr\lp I_d+H^{-\top}H\rp
        \lp\sum_{i=0}^{k-1}\Delta t_iC_i^{\frac12}\rp^2,
    \end{align*}
    where we obtained the second line
    with similar arguments as in the calculus
    of the bound of $I^r_{k,1}$ above.
        
    \emph{Third step: getting the desired bound.}
    Using the triangular inequality on the norm induced by $S$,
    we obtain
    \begin{align*}
        \EE\lb(\bar{\eta}_k)^{\top}
        S
        \bar{\eta}_k\rb
        &\leq
        \lp
        \sum_{r=0}^{k-1}
        \EE\lb(\bar{\eta}^r_k)^{\top}
        S
        \bar{\eta}^r_k\rb^{\frac12}
        \rp^2
        \\
        &\leq
        2\EE\lb(\bar{\eta}^r_0)^{\top}
        S
        \bar{\eta}^r_0\rb
        +2\lp
        \sum_{r=1}^{k-1}
        \EE\lb(\bar{\eta}^r_k)^{\top}
        S
        \bar{\eta}^r_k\rb^{\frac12}
        \rp^2
        \\
        &\leq
        \frac{C}{\aa k}
        +\frac{C}{k^2(1-\aa^{\frac12}R)}
        \tr(I_d+HH^{-\top})
        \lp \sum_{i=0}^{k-1}C_i
        + \lp\sum_{i=0}^{k-1}\Delta t_iC_i^{\frac12}\rp^2\rp.
    \end{align*}
    Therefore, if
    $\sum_{k=0}^{\infty} \Delta t_k^2$ is finite,
    then $C_k$ is uniformly bounded and we can conclude
    by taking $\ee=\Delta t_0^{-2}$.
    If instead
    $\sum_{i=0}^{k-1} \Delta t_i^2\leq a\ln(1+k)$,
    we obtain that $C_k\leq (1+k)^{a\ee}$
    and $\sum_{i=0}^{k-1}C_i$ is of order $k^{1+a\ee}$
    leading to the desired inequality up to changing $\ee$
    into $a^{-1}\ee$.
\end{proof}

    \subsection{Proof of Theorem \ref{thm:simple_cvg_RG}}
    \label{proof:simple_cvg_RG}
    Let us start by proving the following theorem
    on stochastic gradient descent methods.
    \begin{theorem}
        \label{thm:cvg_SGD}
        Let $f:\Theta\to\RR$ be $\mu$-convex,
        $L$-semi-concave,
        and such that $\th^*=\argmin_{\th}f(\th)$
        satisfies $|\th^*|\leq M$ for some $M>0$.
        For $\th_0\in\Theta$,
        the sequence $(\th_k)_{k\geq0}$
        is defined by induction
        using the following projected stochastic gradient
        descent method,
        \begin{equation*}
            \th_{k+1}
            =
            \Pi_{B(0,M)}\lp\th_k-\aa_kg_k\rp,
        \end{equation*}
        for $k\geq0$,
        where 
        $\aa_k>0$
        is convergent to zero, and $\sum_{k\geq0}\aa_k=\infty$,
        $\labs\EE\lb g_k|\th_k\rb-f'(\th_k)\rabs \leq (1+|\th_k|)\ee_k$,
        $\ee_k\in\RR_+$ is convergent to zero,
        and $\EE\lb|g_k|^2\big|\th_k\rb\leq C(1+|\th_k|^2)$.
        Then $(\th_k)_{k\geq0}$ is convergent in expectation
        to $\th^*$, and
        \begin{equation*}
            \EE\lb|\th_k-\th^*|^2\rb
            \leq
            4M^2e^{-\frac{\mu}2\sum_{i=0}^{k-1}\aa_i}
            +C(1+M^2)
            \sum_{i=0}^{k-1}
            \aa_i\lp\aa_i+\mu^{-1}\ee_i^2\rp
            e^{-\frac{\mu}2\sum_{j=i+1}^{k-1}\aa_j}.
        \end{equation*}
    \end{theorem}

    \begin{proof}
        Up to starting the
        iterative algorithm from $\th_1$ instead of $\th_0$,
        we may assume that $|\th_k|\leq M$ for every $k\geq0$.
        For $k\geq0$,
        let us denote $b_k=|\th_k-\th^*|^2$.
        We recall that
        $|\Pi_{B(0,M)}(\th)-\th^*|\leq|\th-\th^*|$
        for any $\th\in\Theta$,
        since $\th^*\in B(0,M)$.
        This and the induction relation satisfied
        by $\th_k$, imply
        \begin{align*}
            b_{k+1}
            &\leq
            \EE\lb\labs\th_k-\th^*-\aa_kg_k\rabs^2\rb
            \\
            &\leq
            b_k
            -2\aa_k\EE\lb(\th_k-\th^*)^{\top}g_k\rb
            +\aa_k^2\EE\lb \labs g_k\rabs^2\rb
            \\
            &\leq
            b_k
            -2\aa_k\EE\lb(\th_k-\th^*)^{\top}\EE\lb g_k|\th_k\rb\rb
            +\aa_k^2\EE\lb\EE\lb \labs g_k\rabs^2|\th_k\rb\rb
            \\
            &\leq
            b_k
            -2\aa_k\EE\lb(\th_k-\th^*)^{\top}f'(\th_k)\rb
            +2\aa_k\ee_k\EE\lb|\th_k-\th^*|(1+|\th_k|)\rb
            +C\aa_k^2\EE\lb(1+|\th_k|^2)\rb
            \\
            &\leq
            b_k
            -2\aa_k\EE\lb f(\th^*)-f(\th_k)
            -\frac{\mu}2|\th_k-\th|^2\rb
            +2(1+M)\aa_k\ee_k\EE\lb|\th_k-\th^*|\rb
            +C(1+M^2)\aa_k^2
            \\
            &\leq
            (1-\mu\aa_k)b_k
            +\frac{\mu}2\aa_k\EE\lb|\th_k-\th^*|^2\rb
            +4(1+M^2)\mu^{-1}\aa_k\ee_k^2
            +C(1+M^2)\aa_k^2
            \\
            &\leq
            (1-\frac{\mu}2\aa_k)b_k
            +C(1+M^2)\aa_k(\mu^{-1}\ee_k^2+\aa_k)
            \\
            &\leq
            e^{-\frac{\mu}2\aa_k}b_k
            +C(1+M^2)\aa_k(\mu^{-1}\ee_k^2+\aa_k),
        \end{align*}
        where
        we used
        the $\mu$-strong convexity of $f$
        to get to the fifth line, and
        a Young inequality to obtain the sixth line.
        Therefore,
        we obtain
        \begin{equation*}
            b_k
            \leq
            e^{-\frac{\mu}2\sum_{i=0}^{k-1}\aa_i}
            b_0
            +C(1+M^2)
            \sum_{i=0}^{k-1}
            \aa_i\lp\aa_i+\mu^{-1}\ee_i^2\rp
            e^{-\frac{\mu}2\sum_{j=i+1}^{k-1}\aa_j},
        \end{equation*}
        which leads to the desired inequality
        using $b_0\leq (|\th_0|+|\th^*|)^2\leq4M^2$.
    \end{proof}

    \begin{proof}[Proof of Theorem \ref{thm:simple_cvg_RG}]
        This proof consists in checking that we can apply
        Theorem \ref{thm:cvg_SGD},
        using the following notations,
        \begin{equation*}
            f(\th)
            =
            \EE\lb\labs\cL v(X,\th)\rabs^2\rb
            +\frac{\ss^4}2\EE\lb\tr\lp D^2_xv(X,\th)^2\rp\rb
            +\frac{\mu}2|\th|^2,
            \;\text{ and }
            g_k
            =
            \nabla_{\th}|\delta_k|^2
            +\mu\th_k.
        \end{equation*}
        Thus, we get,
        \begin{align*}
            \EE\lb g_k|\th_k\rb
            &=
            \EE\lb
            \nabla_{\th}
            \labs
            \cL(X_k,\th_k)
            +R_{0,k}^{\top}\th_k
            +\Delta t_k^{\frac12}
            R_{1,k}^{\top}\th_k
            +\Delta t_k
            R_{2,k}^{\top}\th_k
            \rabs^2
            \rb
            +\mu\th_k
            \\
            &=
            \!\begin{multlined}[t][10.5cm]
            \EE\lb
            \nabla_{\th}
            \labs
            \cL(X_k,\th_k)
            \rabs^2\rb
            +\EE\lb
            \nabla_{\th}
            \labs
            R_{0,k}^{\top}\th_k
            \rabs^2\rb
            +\mu\th_k
            +\Delta t_k\EE\lb
            \nabla_{\th}
            \labs
            R_{1,k}^{\top}\th_k
            \rabs^2\rb
            +2\Delta t_k\EE\lb
            \nabla_{\th}
            \lp\ddt_k
            R_{2,k}^{\top}\th_k
            \rp\rb.
            \end{multlined}
        \end{align*}
        Then,
        from \eqref{eq:cvg_mk},
        we obtain
        \begin{equation*}
            \labs\EE\lb
            \nabla_{\th}
            \labs
            \cL(X_k,\th_k)
            +R_{0,k}^{\top}\th_k
            +\Delta t_k^{\frac12}
            R_{1,k}^{\top}\th_k
            +\Delta t_k
            R_{2,k}^{\top}\th_k
            \rabs^2
            \rb
            +\mu\th_k
            -f'(\th_k)\rabs
            \leq
            C(1+|\th_k|).
        \end{equation*}
        This implies that
        $\labs\EE\lb g_k|\th_k\rb-f'(\th_k)\rabs\leq C\Delta t_k(1+|\th_k|)$.
        The fact that $\EE\lb|g_k|^2\big|\th_k\rb\leq C(1+|\th_k|^2)$
        is straightforward. 
        Theorem \ref{thm:cvg_SGD}
        and the inequalities $|\th^*|^2\leq C\mu^{-1}$
        and $\displaystyle{\exp\lp-\sum_{j=i+1}^k\frac1j\rp}
        \leq i/k$ for $k> i\geq0$,
        conclude the proof.
    \end{proof}

 \section{Technical results}
    \subsection{Proof of Proposition \ref{prop:varred}}
    \begin{proof}[Proof of Proposition \ref{prop:varred}]
        The proof differs whether we assume
        that the observations come from
        a simulator or from the real world
        (see Section \ref{sec:TDs} for precise definitions).

        For observations coming from a simulator,
        it is a direct consequence of Lemma
        \ref{lem:expand_eps} below.

        In the following, we make the proof
        in the case of real-world observations
        for $\ddt$
        (the proof for $\dd$ is similar).

    In this proof, the dependence
    of $v$ in $\th$ is omitted.
    From Itô calculus,
    we have,
    \begin{equation*}
        v(X_{\Delta t})
        =
        v(X_0)
        +\int_0^{\Delta t}
        \lp\nabla_xv(X_t)\cdot b(X_t)
        +\textstyle\frac{\ss^2}2\Delta_x v(X_t)\rp
        dt
        +\ss\int_0^{\Delta t}
        \nabla_xv(X_t)\cdot dW_t.
    \end{equation*}
    Therefore, the continuous
    temporal difference satisfies,
    \begin{multline*}
        \ddt^{\rm cont}_{\Delta t}
        =
        \cL v(x)
        +\frac{1-e^{-\rho\Delta t}-\rho\Delta t}{\Delta t} v(x)
        -\frac{e^{-\rho\Delta t}}{\Delta t}
        \int_0^{\Delta t}
        (\nabla_xv(X_t)\cdot b(X_t)
        -\nabla_xv(x)\cdot b(x))dt
        \\
        -\frac{\ss^2e^{-\rho\Delta t}}{2\Delta t}
        \int_0^{\Delta t}
        (\Delta_xv(X_t)-\Delta_xv(x))dt
        -\frac{\ss e^{-\rho\Delta t}}{\Delta t}
        \int_0^{\Delta t}
        (\nabla_xv(X_t)
        -\nabla_xv(x))\cdot dW_t.
    \end{multline*}
    In the latter equality,
    the last term has zero mean,
    the second is convergent to zero
    and we prove in the following that
    the third and fourth are also
    convergent to zero.

    Take $g:\Omega\to\RR$ a bounded continuous
    function (we take $g=\nabla_xv\cdot b$
    for the proof of the convergence of the third term,
    and $g=\Delta_xv$ for the proof concerning the fourth term).
    We define $A$ as a set of measure zero
    such that 
    $(X_t(\omega))_{0\leq t\leq 1}$ is continuous
    for any $\omega\in\Omega_X\backslash A$
    (where $\Omega_X$ is the sample space of the
    random process $X$).
    For any $\omega\in\Omega_X\backslash A$,
    Heine's Theorem states that
    $t\in[0,1]\to X_t(\omega)$ admits
    a uniform modulus of continuity
    (which depends on $\omega$),
    this implies that
    \begin{equation*}
        \lim_{\Delta t\to 0}
        \frac{1}{\Delta t}
        \int_0^{\Delta t} (g(X_t(\omega))-g(x))dt
        =
        0.
    \end{equation*}
    We just proved that
    $\frac{1}{\Delta t}
    \int_0^{\Delta t} (g(X_t(\omega))-g(x))dt$
    converges almost surely to zero,
    moreover it is uniformly bounded 
    because $g$ is bounded,
    so by the dominated convergence theorem,
    we obtain:
    \begin{equation*}
        \lim_{\Delta t\to 0}
        \EE\lb\frac{1}{\Delta t}
        \int_0^{\Delta t} (g(X_t)-g(x))dt\rb
        =
        0.
    \end{equation*}
    As a consequence, we obtain 
        $\lim_{\Delta t\to0}
        \EE[\ddt^{\rm cont}_{\Delta t}]
        =
        \cL v(x,\th)$.

    Similar arguments imply that
    \begin{equation*}
        \lim_{\Delta t\to 0}
        \EE\lb\frac{1}{\Delta t}
        \int_0^{\Delta t} \labs g(X_t)-g(x)\rabs^2dt\rb
        =
        0,
    \end{equation*}
    so the only term on the right-hand side of
    the latter expansion of $\ddt^{\rm cont}_{\Delta t}$
    whose variance does not
    vanish at the limit
    is the last, i.e.,
    \begin{equation}
        \label{eq:lim_equal}
        \lim_{\Delta t\to0}
        \Var(\ddt^{\rm cont}_{\Delta t})
        =
        \lim_{\Delta t\to0}
        \frac{\ss^2}{\Delta t^2}
        \EE\lb
        \labs\int_0^{\Delta t}
        (\nabla_xv(X_t)
        -\nabla_xv(x))\cdot dW_t\rabs^2\rb.
    \end{equation}

    Using Itô calculus on $\nabla_xv(X_t)$, we obtain
    \begin{equation*}
        \nabla_x v(X_t)
        -\nabla_x v(x)
        =
        \int_0^t
        \lp D^2_xv(X_s)b(X_s)
        +\nabla_x\Delta_xv(X_s)\rp ds
        +\int_0^t
        D^2v(X_s)dW_s.
    \end{equation*}
    Let us prove that
    the first integrable in the latter
    equality leads to a vanishing term
    only in the limit \eqref{eq:lim_equal}.
    This time, we take $g=
        D^2_xv\,b
        +\nabla_x\Delta_xv$,
    let us consider the following sequence of inequalities
    \begin{align*}
        \EE\lb
        \labs\int_0^{\Delta t}
        \int_0^t
        g(X_s)ds\cdot dW_t\rabs^2\rb
        =
        \int_0^{\Delta t}
        \EE\lb
        \labs
        \int_0^t
        g(X_s)ds\rabs^2\rb dt
        \leq
        \int_0^{\Delta t}
        t^2\norminf[2]{g} dt
        =
        \frac{\Delta t^3}3\norminf[2]{g},
    \end{align*}
    Indeed, once we multiply by $\frac{\ss^2}{\Delta t^2}$,
    this leads to a term of order
    $\Delta t$ which will vanish at the limit $\Delta t\to0$.
    Let us consider the only remaining part of the
    variance,
    \begin{align*}
        \frac{\ss^2}{\Delta t^2}
        \EE\lb
        \labs\int_0^{\Delta t}
        \int_0^t
        \ss
        D^2_xv(X_s)dW_s\cdot dW_t\rabs^2\rb
        &=
        \frac{\ss^4}{\Delta t^2}
        \int_0^{\Delta t}
        \EE\lb
        \labs
        \int_0^t
        D^2_xv(X_s)dW_s\rabs^2\rb dt
        \\
        &=
        \frac{\ss^4}{\Delta t^2}
        \int_0^{\Delta t}
        \int_0^t
        \EE\lb
        \tr\lp D^2_xv(X_s)^2\rp\rb ds\, dt
        \\
        &=
        \frac{\ss^4}{\Delta t^2}
        \int_0^{\Delta t}
        (\Delta t-s)
        \EE\lb
        \tr\lp D^2_xv(X_s)^2\rp\rb ds
        \\
        &=
        \ss^4
        \int_0^{1}
        (1-u)
        \EE\lb
        \tr\lp D^2_xv(X_{u\Delta t})^2\rp\rb du,
    \end{align*}
    where the last line is obtained using
    the change of variable $s=u\,\Delta t$.
    Using once again the dominated convergence theorem,
    we obtain:
    \begin{equation*}
        \lim_{\Delta t\to0}
        \Var(\ddt^{\rm cont}_{\Delta t})
        =
        \ss^4
        \EE\lb
        \tr\lp D^2_xv(x)^2\rp\rb
        \int_0^1(1-u)du
        =
        \frac{\ss^4}{2}
        \EE\lb
        \tr\lp D^2_xv(x)^2\rp\rb.
    \end{equation*}
    This concludes the proof.
\end{proof}

    \subsection{Expansions of the temporal differences}
    
    \begin{lemma}
        \label{lem:expand_eps}
        Assume that the observations
        come from a simulator as defined
        in Section \ref{sec:TDs},
        for $(x,\xi,\th)\in\Omega\times\RR^d\times\Theta$
        and $0<\Delta t<1$,
        we have
        \begin{align*}
            \ddt_{\Delta t}(x,S_{\Delta t}(x,\xi),\th)
            &=
            \cL v(x)
            +R_0(x,\xi)^{\top}\th
            +\Delta t^{\frac12}R_1(x,\xi)^{\top}\th
            +\Delta tR_2(\Delta t,x,\xi)^{\top}\th
            \\
            R_0(x,\xi)^{\top}\th
            &=
            \frac{\ss^2}2\lp\Delta_xv(x)
            -\xi^{\top}D^2_xv(x)\xi\rp,
            \\
            R_1(x,\xi)^{\top}\th
            &=
            \rho\ss\nabla_xv(x)\cdot\xi
            -\frac{\ss}2b(x,u(x))^{\top}
            D_x^2v(x)\xi
            -\frac{\ss^3}6d^3_xv(x)(\xi,\xi,\xi),
        \end{align*}
        for some $R_2(\Delta t,x,\xi)$ such that,
        if $\xi$ is a random variable
        normally distributed with zero mean and identity covariance matrix,
        then for $p\geq1$, $\EE\lb|R_2(\Delta t,x,\xi)|^p\rb$
        is bounded uniformly with respect to $\Delta t$ and $x$.
    \end{lemma}
    A similar result with real-world observations
    can be derived with additional
    terms in $R_0$ and $R_1$ which do depend on $\Delta t$
    but vanish when $\Delta t$
    tends to zero.
    Its proofs is a direct consequence of the proof of Proposition
    \ref{prop:varred} in the previous section.
    \begin{proof}
        The proof consists in defining
        $\vp:[0,1]\to\RR$ by
        \begin{equation*}
            \vp(s)
            =
            e^{-s\rho\Delta t}
            v\lp x
            +s\lp
            b(x,u(x))\Delta t
            +\ss\sqrt{\Delta t}\xi
            \rp\rp,
        \end{equation*}
        and taking the development up to order four,
        \begin{equation*}
            \vp(1)
            =
            \vp(0)
            +\vp'(0)
            +\frac{\vp''(0)}2
            +\frac{\vp'''(0)}6
            +\int_0^1
            \frac{(1-s)^3}6
            \vp''''(s)ds.
        \end{equation*}
        Using $\bt\in\RR^d$ defined
        by $\bt=b(x,u(x))\Delta t +\ss\sqrt{\Delta t}\xi$,
        the latter derivatives of $\vp$ are given by
        \begin{align*}
            \vp(0)
            &=
            v(x)
            \\
            \vp'(0)
            &=
            -\rho\Delta tv(x)
            +\nabla_xv(x)\cdot\bt
            \\
            \vp''(0)
            &=
            \rho^2\Delta t^2v(x)
            -2\rho\Delta t\nabla_xv(x)\cdot\bt
            +d^2_xv(x)(\bt,\bt)
            \\
            \vp'''(0)
            &=
            -\rho^3\Delta t^3v(x)
            +3\rho^2\Delta t^2\nabla_xv(x)\cdot\bt
            -3\rho\Delta td^2_xv(x)(\bt,\bt)
            +d^3_xv(x)(\bt,\bt,\bt)
            \\
            \vp''''(s)
            &=
            e^{-s\rho\Delta t}
            \lb
            \rho^4\Delta t^4v
            -4\rho^3\Delta t^3
            \nabla_xv\cdot\bt
            +6\rho^2\Delta t^2
            d^2_xv(\bt,\bt)
            -4\rho\Delta t
            d^3_xv(\bt,\bt,\bt)
            +d^4v(\bt,\bt,\bt,\bt)
            \rb.
        \end{align*}
        We conclude by replacing 
        all the equalities in this proof
        in \eqref{eq:def_ddt}.
    \end{proof}

\subsection{Some lemmas used in the proof
    of Theorem \ref{thm:averaged_cvg}}

\begin{lemma}
    \label{lem:HSA}
    The matrices
    $S$ and $A$ are respectively
    the symmetric and asymmetric part of $H$.
    Moreover, they satisfy
    \begin{align}
        \label{eq:bound_S2}
        S^2
        &\leq
        \tr(S)S
        \\
        \label{eq:bound_A2}
        A^{\top}A
        &=
        -A^2
        \leq
        \frac{2}{\rho\ss^2}
        \norminf[2]{b+\frac{\ss^2}2\nabla_x \ln(m)}
        S^2
        \\
        \label{eq:bound_SA_AS}
        (SA-AS)
        &\leq
        2\sqrt{\frac{2}{\rho\ss^2}}
        \norminf{b+\frac{\ss^2}2\nabla_x\ln(m)}S^2,
        \\
        \label{eq:boundHX2}
        \EE\lb H(X)H(X)^{\top}\rb
        &\leq
        \rho^{-1}
        \norminf[2]{\cL\vp(X)}S.
    \end{align}
\end{lemma}
\begin{proof}
    \emph{First step: proving that $S$ 
        and $A$ are respectively the symmetric
        and asymmetric part of $H$.}
    Take $\th\in\Theta$, we get:
    \begin{align*}
        \th^{\top}
        H\th
        &=
        \th^{\top}\EE\lb \vp(X)\cL\vp(X)^{\top}\rb\th
        \\
        &=
        \EE\lb v(X,\th)\cL v(X,\th)\rb
        \\
        &=
        \int_{\Omega}
        \lp \rho v
        -\frac{\ss^2}2\Delta_xv
        +b(x)\cdot\nabla_xv\rp v(x)
        m(x)dx
        \\
        &=
        \rho\EE\lb v(X)^2\rb
        +\frac{\ss^2}2\EE\lb \labs\nabla_xv(X)\rabs^2\rb,
    \end{align*}
    where the last line is obtained by using
    the fact that $m$ satisfies
    \begin{equation}
        \label{eq:PDEm}
        -D^2_{x,x}\cdot\lp\frac{\ss\ss^{\top}}2m\rp
        +\divo(bm)
        =
        0,
    \end{equation}
    and the
    following integration by parts,
    \begin{align*}
        \int_{\Omega}
        \nabla_xv\cdot b(x) v(x)m(x)dx
        &=
        \int_{\Omega}
        \frac12\nabla_x(v^2)\cdot b(x)m(x)dx
        \\
        &=
        -\frac12
        \int_{\Omega}
        \divo(b(x)m(x))v^2(x)dx,
        \\
        -\int_{\Omega}
        \Delta_x v(x) v(x)m(x)dx
        &=
        \int_{\Omega}
        \labs\nabla_x v\rabs^2m(x)dx
        +\int_{\Omega}
        \frac12
        \nabla_x(v^2)\cdot\nabla_xm(x)dx
        \\
        &=
        \int_{\Omega}
        \labs\nabla_x v\rabs^2m(x)dx
        -\frac12\int_{\Omega}
        \Delta_xm(x)v^2(x)dx.
    \end{align*}
    This implies that $S$ is the symmetric part
    of $H$. Then it is straightforward that the
    asymmetric part of $H$ is equal to $A$.

    \emph{Second step: proving the four inequalities.}
    The first inequality \eqref{eq:bound_S2}
    is straightforward,
    it only relies on the fact that
    $S$ is symmetric and positive.
    The fourth inequality \eqref{eq:boundHX2}
    is straightforward using the definitions
    of $H(X)$ and $S$.
    The third inequality \eqref{eq:bound_SA_AS}
    is a consequence of \eqref{eq:bound_A2}.
    Therefore, there is only \eqref{eq:bound_A2}
    left to prove.
    Let us take $\ll\in\CC$ a complex
    eigenvalue of $H$,
    and $\th$ an associated normalised eigenvector,
    it satisfies
    $\bar{\th}^{\top}S\th=\Re(\ll)$ and
    $\bar{\th}^{\top}A\th=i\Im(\ll)$.
    Therefore, we get
    \begin{align*}
        |\Im(\ll)|
        &=
        \labs \bar{\th}^{\top}A\th\rabs
        \\
        &=
        \labs\EE\lb \bar{v}(X,\th)
        \lp b(X)
        +\nabla_x\ln m(X)\rp^{\top}
        \nabla_xv(X,\th)\rb\rabs
        \\
        &\leq
        \norminf{b+\nabla_x\ln(m)}
        \EE\lb|v(X,\th)|^2\rb^{\frac12}
        \EE\lb|\nabla_xv(X,\th)|^2\rb^{\frac12}
        \\
        &\leq
        \sqrt{\frac{2}{\rho\ss^2}}
        \norminf{b+\nabla_x\ln(m)}
        \bar{\th}^{\top}S\th.
    \end{align*}
    This concludes the proof.
\end{proof}

\begin{lemma}
    \label{lem:bound_I_gH}
    For $\aa\leq R^{-2}$,
    the following two inequalities hold
    for any $k\geq 0$,
    \begin{align}
        \label{eq:bound_I_gH1}
        (I_d-\aa H^{\top})
        (I_d-\aa H)
        &\leq
        I_d-\aa S
        \\
        \label{eq:bound_I_gH2}
        \lp I_d -(I-\aa H^{\top})^k\rp
        \lp I_d -\lp I-\aa H\rp^k\rp
        &\leq
        \aa^2 k^2H^{\top}H,
        \\
        \label{eq:bound_I_gH3}
        \lp I_d -(I-\aa H^{\top})^k\rp
        \lp I_d -\lp I-\aa H\rp^k\rp
        &\leq
        4\lp 1+
        \frac{2}{\rho\ss^2}
        \norminf[2]{b+\nabla_x\ln(m)}\rp I_d,
        \\
        \label{eq:bound_I_gH4}
        \lp I_d -(I-\aa H^{\top})^k\rp
        \lp H^{-1}+H^{-\top}\rp
        \lp I_d -\lp I-\aa H\rp^k\rp
        &\leq
        2\aa k\lp 1+
        \sqrt{\frac{2}{\rho\ss^2}}
        \norminf{b+\nabla_x\ln(m)}\rp I_d.
    \end{align}
\end{lemma}
The latter lemma would be straight forward if $H$
were symmetric.
Conversely, it does not hold if we only assume
the eigenvalues of $H$ to be bounded and with positive
real part.
In fact, we need some bound on the imaginary
part of the spectrum of $H$, depending on its
real part.

\begin{proof}
    One may notice that
    \eqref{eq:bound_I_gH4} is a straightforward
    consequence of \eqref{eq:bound_I_gH2}
    and \eqref{eq:bound_I_gH3}.
    Then, concerning \eqref{eq:bound_I_gH1},
    it is sufficient to write
    $(I_d-\aa H^{\top})
    (I_d-\aa H)
    = I_d-2\aa S +\aa^2\lp S^2+SA-AS-A^2\rp$,
    and use the definition of $R$,
    \eqref{eq:bound_S2},
    \eqref{eq:bound_A2} and
    \eqref{eq:bound_SA_AS}.
    Therefore, it only remains to prove 
    \eqref{eq:bound_I_gH2}
    and \eqref{eq:bound_I_gH3}.

    \emph{First step: proving \eqref{eq:bound_I_gH2}.}
    Let us proceed by induction,
    the case $k=0$ is straightforward.
    Let us denote $y_k=\lp I_d-\lp I_d-\aa H\rp^k\rp$
    and assume that the inequality holds for $k$.
    One may notice that
    for $\th\in\RR^d$,
    using \eqref{eq:bound_I_gH1},
    we obtain
    \begin{align*}
        \th^{\top}y_k^{\top}(I_d-\aa H)^{\top}H\th
        &\leq
        \lp\th^{\top}y_k^{\top}(I_d-\aa H)^{\top}
        (I_d-\aa H)y_k\th\rp^{\frac12}
        \lp \th^{\top}H^{\top}H\th\rp^{\frac12}
        \\
        &\leq
        \aa k
        \th^{\top}H^{\top}H\th,
    \end{align*}
    which implies 
    $y_k^{\top}\lp I_d-\aa H\rp^{\top}H
    +H^{\top}\lp I_d-\aa H\rp y_k
    \leq
    2\aa k H^{\top}H$.
    Using the latter inequality,
    the relation
    $y_{k+1}=(I_d-\aa H)y_k+\aa H$,
    and \eqref{eq:bound_I_gH1} again,
    we get
    \begin{align*}
        y_{k+1}^{\top}y_{k+1}
        &=
        y_k^{\top}(I_d-\aa H)^{\top}(I_d-\aa H)y_k
        +\aa y_k^{\top}(I_d-\aa H)^{\top}H
        +\aa H^{\top}(I_d-\aa H)y_k
        +\aa^2 H^{\top}H
        \\
        &\leq
        \aa^2 k^2 H^{\top}H
        +2\aa^2 k H^{\top}H
        +\aa^2 H^{\top}H
        =
        \aa^2(k+1)^2H^{\top}H.
    \end{align*}
    This concludes the induction.

    \emph{Second step: proving \eqref{eq:bound_I_gH3}.}
    In this step, we will only work with
    the complex eigenvalues of $H$:
    let $\ll\in\CC$ be one of them,
    we get
    \begin{align*}
        \labs 1-(1-\aa \ll)^{k+1}\rabs
        &=
        \labs\lp1-\aa\ll\rp \lp1-(1-\aa \ll)^{k}\rp+\aa\ll\rabs
        \\
        &\leq
        \lp\labs1-\aa\ll\rabs
        \labs 1-(1-\aa\ll)^k\rabs
        +\aa|\ll|\rp.
    \end{align*}
    This implies
    \begin{align*}
        \labs 1-(1-\aa \ll)^{k}\rabs
        &\leq
        \aa|\ll|\sum_{j=0}^{k-1}
        |1-\aa\ll|^j
        \\
        &\leq
        \frac{\aa|\ll|}{1-|1-\aa\ll|}
        \\
        &\leq
        \frac{\aa|\ll|}{1-(1-\aa\Re(\ll))^{\frac12}}
        \;\text{ using \eqref{eq:bound_I_gH1},}
        \\
        &\leq
        \frac{\aa|\ll|}{1-(1-\frac{\aa}2\Re(\ll))}
        \;\text{ because }\aa\Re(\ll)\leq 1,
        \\
        &\leq
        2\sqrt{1+\frac{\Im(\ll)^2}{\Re(\ll)^2}}
        \\
        &\leq
        2\lp 1+
        \frac{2}{\rho\ss^2}
        \norminf[2]{b+\nabla_x\ln(m)}\rp^{\frac12},
    \end{align*}
    where the last inequality comes from a similar
    argument as in the proof of \eqref{eq:bound_A2}.
    This concludes the proof.
\end{proof}

\begin{lemma}
    \label{lem:cov_vp}
    Assume \ref{hypo:vp_reg}.
    There exists $C>0$ such that
    the two following inequalities hold
    for any $k\geq0$,
    \begin{align*}
        \EE\lb\vp(X_k)\otimes\vp(X_k)\rb
        &\leq
        CS,
        \\
        (\EE\lb H(X_k)\rb-H)
        (\EE\lb H(X_k)\rb-H)^{\top}
        &\leq
        C\Delta t_k^2S.
    \end{align*}
\end{lemma}

\begin{proof}
    We recall that the set of admissible functions
    $v$ is finitely dimensional,
    therefore the $C^4$-norm and the $H^1(m)$-norm
    are equivalent and there exists $C>0$
    such that 
    $\norm[2]{v(\cdot,\th)}{C^4}
    \leq C
    \norm[2]{v(\cdot,\th)}{H^1(m)}$.
    For $\th\in\Theta$
    and $k\geq0$,
    this implies
    \begin{align*}
        \th^{\top}
        \EE\lb\vp(X_k)\otimes\vp(X_k)\rb\th
        &=
        C\EE\lb v(X_k,\th)^2\rb
        \\
        &\leq
        C\EE\lb v(X,\th)^2\rb
        +C\Delta t_k\norm{v(\cdot,\th)^2}{C^4}
        \\
        &\leq
        C\lp 1+\Delta t_k\rp
        \norm[2]{v(\cdot,\th)}{H^1(m)},
    \end{align*}
    where the second line is obtained from Theorem
    \eqref{eq:cvg_mk}.
    Here,
    $C$ is a constant that can change from line to line.
    The first inequality is then obtained by recalling that
    $\norm[2]{v(\cdot,\th)}{H^1(m)}\leq (\rho^{-1}+2\ss^{-2})\th^{\top}S\th$.

    Concerning the second inequality,
    we get
    \begin{align*}
        \labs\lp\EE\lb H(X_k)\rb-H\rp\th\rabs^2
        &=
        \labs\EE\lb \vp(X_k)\cL v(X_k,\th)-\vp(X)\cL v(X,\th)\rb\rabs^2
        \\
        &\leq
        C\lp\Delta t_k \norm{v(\cdot,\th)}{C^6}\rp^2
        \\
        &\leq
        C\Delta t_k^2
        \norm[2]{v(\cdot,\th)}{H^1(m)},
    \end{align*}
    where the second line is obtained from \eqref{eq:cvg_mk},
    and the third line from the fact that the $C^6$-norm
    is equivalent to the $H^1(m)$ on the finite dimensional
    space of functions $v$.
    We conclude the same way as we did for the first inequality.
\end{proof}

       \subsection{Calculus of variances and covariances}
    \begin{lemma}
        \label{lem:variances}
        Let $(x,\th)\in\Omega\times\Theta$
        and $\xi$ a Gaussian vector
        with zero mean and identity covariance matrix,
        the following equalities hold
        \begin{align}
            \label{eq:var_nabla}
            \Var\lp\xi\cdot\nabla_xv(x)\rp
            &=
            \labs\nabla_xv(x)\rabs^2,
            \\
            \label{eq:var_delta}
            \Var\lp\xi^{\top}D^2v(x)\xi - \Delta_xv(x)\rp
            &=
            2\tr\lp D^2_xv(x)^2\rp.
        \end{align}
    \end{lemma}

    \begin{proof}
        The first equality is straightforward.
        Since $D^2v(x)$ is symmetric,
        there exists $P$ an orthogonal matrix
        and $D$ a diagonal matrix such that
        $D^2v(x)=P^{\top}DP$.
        The couples $(X,\xi)$
        and $(X,P^{\top}\xi)$
        have the same law
        and $\xi$ is independent of $X$ and $D$,
        this implies
        \begin{align*}
            \Var\lp\xi^{\top}D^2v(x)\xi - \Delta_xv(x)\rp
            &=
            \EE\lb\lp\xi^{\top}D^2v(x)\xi - \Delta_xv(x)\rp^2\rb
            \\
            &=
            \EE\lb\lp \lp P^{\top}\xi\rp^{\top}
            D^2v(x)P^{\top}\xi - \Delta_xv(x)\rp^2\rb
            \\
            &=
            \EE\lb\lp \xi^{\top}D\xi
            - \Delta_xv(x)\rp^2\rb
            \\
            &=
            \EE\lb\sum_{i=1}^d D_i^2\lp\xi_i^2-1\rp^2\rb
            =2\sum_{i=1}^d D_i^2
            =2\tr\lp D^2_xv(x)^2\rp.
        \end{align*}
        This concludes the proof.
    \end{proof}

  \subsection{Counterpart to Lemma \ref{lem:expand_eps} in the multi-step setting}
    \begin{lemma}
        \label{lem:expand_eps_n_sq}
        There exists $C>0$
        such that,
        for any $(x,\th)\in\Omega\times\Theta$,
        $n\geq1$,
        $0<\Delta t<\frac1n$
        and $\xi=\lp\xi_i\rp_{0\leq i<n}$ independent  normally distributed
        random variables with zero mean and identity covariance matrix,
        we have
        \begin{align*}
            &\labs\EE\lb|\ddt^n_{\Delta t}(x,\xi,\th)|^2\rb
            -\cL v(x)^2\rabs
            \leq
            C\lp1+\labs\th\rabs^2\rp
            \lp n^{-1} + n\Delta t\rp,
            \\
            &\labs\EE\lb\nabla_{\th}
            |\ddt^n_{\Delta t}(x,\xi,\th)|^2\rb
            -\nabla_{\th}\cL v(x)^2\rabs
            \leq
            C\lp1+\labs\th\rabs\rp
            \lp n^{-1} + n\Delta t\rp.
        \end{align*}
    \end{lemma}
    \begin{proof}
        Taking $X_0=x$
        and $X_{t_{i+1}}=S_{\Delta t}(X_{t_i},\xi_i)$
        for $0\leq i< n$,
        we obtain
        \begin{equation}
            \label{eq:eps_n_from_eps}
            \ddt^n_{\Delta t}(x,\xi,\th)
            =
            \frac1n
            \sum_{i=0}^{n-1}
            \ddt_{\Delta t}(X_{t_i},X_{t_{i+1}},\th).
        \end{equation}
        Let us do the expansion
        of $\cL v(X_{t_i})$
        around $x$ up to order two,
        \begin{equation*}
            \cL v(X_{t_i})
            =
            \cL v(x)
            +\nabla_x\cL v(x)
            \cdot
            \bt
            +\int_0^1(1-s)
            \bt^{\top}D^2_x\cL v\lp x+s\bt\rp\bt ds,
        \end{equation*}
        where $\bt=\sum_{j=0}^{i-1}\lp b(X_{t_j},u(X_{t_j}))\Delta t
        +\ss\sqrt{\Delta t}\xi_j\rp$.
        The latter equalities and Lemma \ref{lem:expand_eps}
        imply
        \begin{multline*}
            \ddt^n(x,\xi,\th)
            =
            \cL(x)
            +\frac{\ss^2}{2n}
            \sum_{i=0}^{n-1}
            \lp\Delta_xv(X_{t_i})
            -\xi_i^{\top}D^2v(X_{t_i})\xi_i\rp
            +\frac1{n\sqrt{\Delta t}}
            \sum_{i=0}^{n-1}
            \Bigl[
            (n-1-i)\ss\nabla_x\cL v(X_{t_i})\cdot\xi_i
            \\
            +\rho\ss\nabla_xv(X_{t_i})\cdot\xi_i
            -\frac{\ss}2b(X_{t_i},u(X_{t_i}))^{\top}
            D^2v(X_{t_i})\xi_i
            -\frac{\ss^3}6d^3_xv(X_{t_i})(\xi_i,\xi_i,\xi_i)
            \Bigr]
            +R^n_{\Delta t}(x,\xi,\th),
        \end{multline*}
        with $\EE\lb\labs R^n_{\Delta t}(x,\xi,\th)\rabs^2\rb
        \leq C(1+|\th|^2)(n^{-1}+n\Delta t)^2$.
        We conclude by taking the expectation of the square
        in the latter equality
        and using the independence of $\lp\xi_i\rp_{0\leq i<n}$.
    \end{proof}

\end{document}